\documentclass[11pt]{article}
\usepackage{hyperref}

\usepackage[final]{acl}

\usepackage[table]{xcolor}
\usepackage{booktabs}
\usepackage{tcolorbox}
\usepackage{enumitem}
\usepackage{times}
\usepackage{latexsym}
\usepackage{multirow}
\usepackage{fontawesome5}
\usepackage{dblfloatfix}
\usepackage[T1]{fontenc}
\usepackage[utf8]{inputenc}
\usepackage{url}
\usepackage{microtype}
\usepackage{amsthm}
\newtheorem{theorem}{Theorem}
\usepackage{amsmath}
\usepackage{amsfonts}
\usepackage{inconsolata}
\usepackage{graphicx}
\usepackage{subfig}
\usepackage{pifont}
\usepackage{xspace}

\definecolor{lightblue}{rgb}{0.796, 0.894, 0.9808}
\definecolor{naturegreen}{RGB}{0,102,85}
\definecolor{naturegray}{RGB}{248,248,248}
\definecolor{crimson}{rgb}{0.863,0.078,0.235}

\newcommand{\ours}{\textsc{StructLoRA}\xspace}

\title{Not All Directions Matter: Towards Structured and Task-Aware\\Low-Rank Model Adaptation}

\author{
  \textbf{Xi Xiao\textsuperscript{1}}\thanks{\ Equal contribution.},
 \textbf{Chenrui Ma\textsuperscript{2}}\footnotemark[\value{footnote}],
 \textbf{Yunbei Zhang\textsuperscript{3}}\footnotemark[\value{footnote}],
 \textbf{Chen Liu\textsuperscript{4}},
 \textbf{Zhuxuanzi Wang\textsuperscript{1}},\\
 \textbf{Yanshu Li\textsuperscript{5}},
 \textbf{Lin Zhao\textsuperscript{6}},
 \textbf{Guosheng Hu\textsuperscript{7}},
 \textbf{Tianyang Wang\textsuperscript{1}}\thanks{\ Corresponding authors.},
 \textbf{Hao Xu\textsuperscript{8}}\footnotemark[\value{footnote}]
\\
{\small
 \textsuperscript{1}University of Alabama at Birmingham,
 \textsuperscript{2}University of Virginia,
 \textsuperscript{3}Tulane University,}
\\
{\small
 \textsuperscript{4}Yale University,
 \textsuperscript{5}Brown University,
 \textsuperscript{6}Northeastern University,
 \textsuperscript{7}University of Bristol,
 \textsuperscript{8}Harvard University
}\\
 \small{Please direct correspondence to
\url{haxu@bwh.harvard.edu} or \url{tw2@uab.edu}.}
 \\
 \small{
 \faGithub\ Project page \url{https://xixiaouab.github.io/StructLoRA/}.}
}

\begin{document}

\maketitle

\begin{abstract}
Low-Rank Adaptation~(LoRA) has become a cornerstone of parameter-efficient fine-tuning~(PEFT). Yet, its efficacy is hampered by two fundamental limitations: \textit{semantic drift}, arising from treating all update directions with equal importance, and \textit{structural incoherence}, due to adapting layers independently, resulting in uncoordinated and suboptimal updates. To address these issues, we propose \ours, a framework that tackles both limitations through a principled dual-component design: (1) an Information Bottleneck-guided filter that prunes task-irrelevant directions to mitigate semantic drift, and (2) a lightweight, training-only graph-based coordinator that enforces inter-layer consistency to resolve structural incoherence. Extensive experiments across large language models, vision language models, and vision models (including LLaMA, LLaVA, and ViT) demonstrate that \ours consistently establishes a new state of the art, outperforming not only vanilla LoRA but also advanced dynamic rank allocation and sparsity-based methods. Notably, the gains are particularly pronounced in challenging low-rank and low-data regimes. Crucially, since the proposed modules operate only during training, \ours improves performance with \textbf{zero additional inference cost}, shifting the focus of PEFT from mere parameter compression to a more holistic optimization of information quality and structural integrity.
\end{abstract}

\vspace{0pt}
\section{Introduction}

\begin{figure}[t]
  \centering
  \includegraphics[width=\columnwidth]{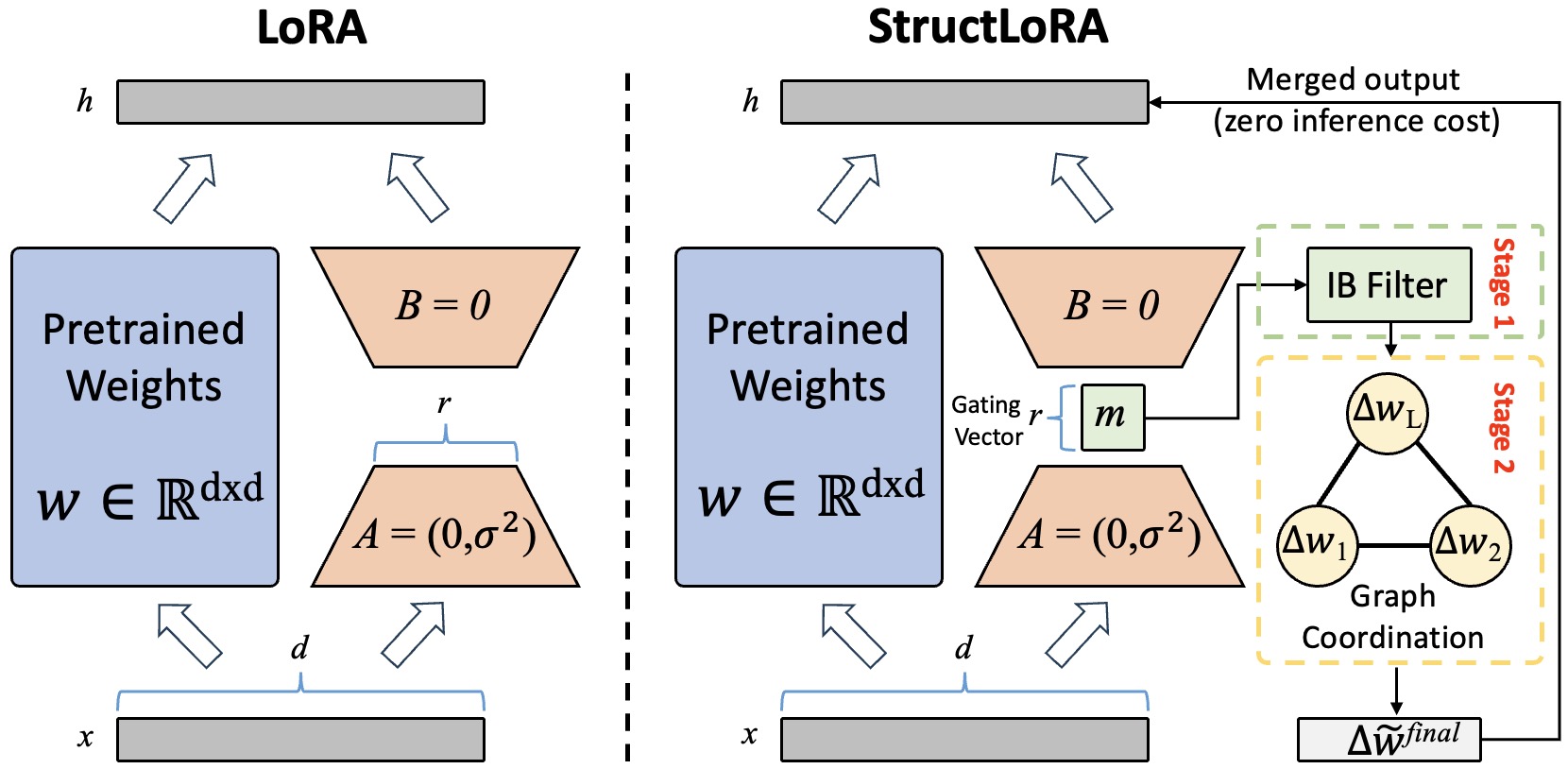}
  \caption{
    \textbf{Architectural comparison between LoRA and \ours.}
    The left illustrates the standard LoRA architecture with uniform low-rank updates,
    while the right shows our \ours, which introduces an \textit{Information Bottleneck (IB) filter} 
    and a \textit{Graph-based Coordination mechanism}. 
    These modules selectively retain task-relevant update directions and align layer-wise updates through message passing. Both operate only during training and are removed at inference, preserving LoRA's zero-latency efficiency.
  }
  \label{fig:lora-vs-structlora}
\end{figure}

The advent of large language models~(LLMs) has revolutionized natural language processing, yet their immense scale renders full fine-tuning prohibitively expensive for adaptation to downstream tasks. Parameter-Efficient Fine-Tuning~(PEFT) has emerged as the \textit{de facto} solution, enabling customization by updating only a small fraction of a model's parameters~\cite{han2024parameter}. Among these methods, Low-Rank Adaptation~(LoRA)~\cite{hu2022lora} has become a dominant paradigm due to its simplicity, parameter efficiency, and zero-latency inference~\cite{mao2025survey}. Its success has led to a range of variants, such as QLoRA~\cite{dettmers2023qlora} and AdaLoRA~\cite{zhang2023adaptive}, which further improve efficiency through quantization and adaptive rank allocation.

However, despite their widespread adoption, we identify two fundamental and largely unaddressed shortcomings in the LoRA paradigm. The first, \textit{semantic drift}, stems from allocating a limited parameter budget uniformly across all low-rank update directions, implicitly assuming that each direction is equally important for the target task. This assumption overlooks the fact that many directions may capture redundant or noisy signals, leading to inefficient use of capacity and suboptimal performance. The second, \textit{structural incoherence}, arises from adapting each layer independently, disregarding the inherent compositional structure of deep models such as Transformers~\cite{raghu2021vision, hu2022layerwise, touvron2022three, liu2026dispersion}. This layer-wise independence can lead to misaligned update directions across layers. We empirically observe this phenomenon, finding low cosine similarity between the update gradients of adjacent layers (e.g., 0.27--0.41), which indicates a lack of semantic coordination that can hinder generalization~\cite{sankararaman2020impact, jiang2024tracing, carbonnelle2018layer}.

To address these critical gaps, we introduce \ours, a novel framework that enhances LoRA with two synergistic components: task-aware filtering and structural coordination. An overview is shown in Figure~\ref{fig:lora-vs-structlora}. To counteract semantic drift, \ours first employs a principled filtering mechanism inspired by the Information Bottleneck principle~\cite{tishby2015deep, alemi2017deep, liao2024assessing}. This module learns to preserve only those update directions that are maximally informative for the task objective by optimizing a constrained objective.

Subsequently, to mitigate structural incoherence, we explicitly model inter-layer dependencies by representing the network as a computational graph. A lightweight, tailored Graph Neural Network~(GNN) then propagates and refines these filtered updates, encouraging smoother and more aligned adaptation trajectories across the model depth. Crucially, this coordination module operates \textbf{only during training} and is discarded at inference, fully preserving LoRA's zero-latency advantage.

We empirically validate the effectiveness of \ours across a comprehensive suite of benchmarks, fine-tuning powerful large language models, vision language models, and vision models such as LLaMA, LLaVA, and ViT on challenging tasks in natural language understanding, computer vision, and multimodal reasoning. Our results demonstrate that \ours consistently and significantly outperforms strong PEFT baselines, including LoRA, AdaLoRA, and DoRA, particularly in low-resource and low-rank regimes. Furthermore, \ours substantially closes the performance gap to full fine-tuning while incurring negligible training overhead and zero additional inference latency.

In summary, our core contributions include:
\begin{itemize}
    \item To our knowledge, this work is the first to identify and empirically demonstrate the problems of \textit{semantic drift} and \textit{structural incoherence} in LoRA-based methods, for which we propose a principled solution that combines task-aware filtering and structural coordination.
    \item We propose \ours, a lightweight, plug-and-play framework featuring an Information Bottleneck-guided filter to prune noisy update directions and a graph-based coordinator to enforce inter-layer consistency. The coordinator is detached at inference time and therefore incurs no inference overhead.
    \item We conduct extensive experiments on large language, multimodal, and vision models, showing that \ours sets a new state of the art for parameter-efficient fine-tuning, with particularly strong gains in low-rank and few-shot scenarios.
\end{itemize}

\section{Related Work}

\subsection{The Landscape of PEFT}
Adapting large pre-trained models with limited compute has motivated a rich line of PEFT methods~\cite{xiao2026promptbased, xiao2025viapt, xiao2025visual, han2024parameter, peft_survey_a2z_2024, xin2024parameter, peft_survey_braz_2024}. Three families are common: additive adapters that insert small trainable blocks~\cite{houlsby2019series, he2022parallel}, selective updates that tune only a subset of native weights such as biases~\cite{zaken2021bitfit}, and reparameterization approaches that learn a low‑dimensional proxy for the update. Our work follows the third family, whose low‑rank updates merge into the backbone after training and thus add no inference latency. \ours keeps LoRA's interface but introduces two principles that target overlooked issues in this family: task‑aware directional filtering and cross‑layer structural coordination.

\subsection{The Evolution of LoRA}
LoRA~\cite{hu2022lora} has spurred variants along complementary axes. For resource footprint, QLoRA combines LoRA with 4‑bit quantization to reduce memory~\cite{dettmers2023qlora}; parameter sharing further cuts trainables (e.g., VeRA and Tied‑LoRA)~\cite{kopiczko2024vera, renduchintala2024tied}. To relax fixed‑rank budgets, adaptive schemes allocate capacity by SVD‑style importance or learnable rank signals~\cite{zhang2023adaptive, shinwari2025ard}. Sparsity‑oriented lines prune modules or preserve sparsity in already sparse models~\cite{drost2024lora_family, lors2025, splora2024}. Another thread improves update quality: DoRA disentangles magnitude and direction~\cite{liu2024dora}, while orthogonality‑based constraints aim to stabilize learning and reduce interference~\cite{qiu2023controlling, hra2024, zhang2025hyperadalora, wang2025ctr}. These advances decide how much capacity to allocate and how to regularize factors; they rarely ask which directions in the low‑rank subspace are semantically useful for the task or how layer‑wise updates should be coordinated. \ours addresses both by coupling an Information‑Bottleneck filter~\cite{tishby2015deep} with a lightweight, training‑only coordinator that propagates filtered signals across depth.

\section{Methodology}
\label{sec:method}

We present \ours, a light extension to LoRA that learns \emph{what} directions to keep and \emph{how} layers should move together. The design follows two observations: many low-rank directions carry little signal for the target task, and layer-wise updates often drift when trained in isolation. \ours addresses both with a two-step procedure over the low-rank updates at each layer: an information bottleneck filter that selects task-relevant directions, and a graph-based coordinator that aligns updates across depth. 


\subsection{Preliminaries: Low-Rank Adaptation}
\label{sec:prelim}

Given a pretrained weight matrix $\mathbf{W}_{0} \in \mathbb{R}^{d\times k}$, LoRA learns a rank-$r$ update $\boldsymbol{\Delta W}=\mathbf{A}\mathbf{B}$ with $\mathbf{A} \in \mathbb{R}^{d\times r}$ and $\mathbf{B} \in \mathbb{R}^{r\times k}$, while keeping $\mathbf{W}_0$ frozen~\cite{hu2022lora}. The forward becomes Eqn~\eqref{eq:lora-forward}, and the product $\mathbf{A}\mathbf{B}$ can be merged into $\mathbf{W}_0$ after training, so inference cost does not change.
\begin{equation}
\label{eq:lora-forward}
\mathbf{y} = \big(\mathbf{W}_0 + \alpha \mathbf{A}\mathbf{B}\big) \mathbf{x}
\end{equation}
\ours keeps this interface and adds training-time selectivity and coordination.

\subsection{Stage 1: Information Bottleneck-Guided Directional Filtering}
\label{sec:ib-filter}

LoRA spreads a small budget over $r$ directions and treats them the same. In practice, only a few directions help predict the label; others carry nuisance variation. We want the update to retain information about $\mathbf{Y}$ but discard information about $\mathbf{X}$ that does not help the task.

Let $\boldsymbol{\Delta W}=\mathbf{A}\mathbf{B}$. We gate the $r$ rank-one directions with a learnable mask $\mathbf{m} \in [0,1]^r$ and form the filtered update, as shown in Eqn~\eqref{eq:filtered-update}, where $\text{diag}(\mathbf{m})$ denotes a diagonal matrix whose diagonal entries are the elements of $\mathbf{m}$, 
scaling each rank-one direction by its corresponding gate value.
\begin{equation}
\label{eq:filtered-update}
\boldsymbol{\Delta \tilde{W}} = \mathbf{A} \text{diag}(\mathbf{m}) \mathbf{B}
\end{equation}
We learn $\mathbf{m}$ by an information bottleneck (IB) objective~\cite{tishby2015deep, alemi2017deep, liao2024assessing}, as described in Eqn~\eqref{eq:ib-lagrangian}, where $\mathcal{L}_{\text{task}}$ is the supervised loss, and $I(\cdot ; \cdot)$ denotes mutual information. 
\begin{equation}
\label{eq:ib-lagrangian}
\mathcal{L}_{\text{IB}}
=
\mathcal{L}_{\text{task}}
+ \beta\, I\!\big(\boldsymbol{\Delta \tilde{W}};\,\mathbf{X}\big)
- \gamma\, I\!\big(\boldsymbol{\Delta \tilde{W}};\,\mathbf{Y}\big)
\end{equation}
Following the variational IB derivation~\cite{alemi2017deep}, we use a tractable upper bound in which the compression term introduces a KL penalty between a learned posterior over the gate and a simple prior; this acts like a sparsity/weight-decay regularizer on $\mathbf{m}$. When we need hard selection, we use a Gumbel-Softmax relaxation to keep training smooth while pushing $\mathbf{m}$ toward $\{0,1\}$~\cite{jang2017categorical}. If we estimate MI terms $I(\cdot;\cdot)$ explicitly, we adopt standard variational estimators such as MINE~\cite{belghazi2018mutual}. In short, the filter lets LoRA spend rank on the few directions that move the loss.

Eq.~\eqref{eq:ib-lagrangian} raises the signal-to-noise ratio of the update: it rewards dependence on labels and penalizes spurious dependence on inputs. Under small $r$ or small data, this selectivity reduces variance and makes training steadier. We observe that the filtered updates show higher cross-layer consistency and better task scores than the unfiltered ones (see Section~\ref{sec:experiments}).

\subsection{Stage 2: Graph-Based Layer Coordination}
\label{sec:gnn-coord}

Backprop couples all parameters through the loss, but it does not impose a structural prior on how \emph{adjacent} layers should update. In Transformers, features change gradually with depth. If layers move in different directions, the trajectory fragments. We add a light coordinator that lets layers exchange their update signals.

\paragraph{Graph construction}
We view the network as a graph $\mathcal{G}=(\mathcal{V},\mathcal{E})$ with one node per layer. The node feature is the flattened filtered update $\mathbf{h}^{(0)}_{\ell}=\text{vec}(\boldsymbol{\Delta \tilde{W}}_{\ell})$. We connect adjacent layers, and we may add semantic edges between layers whose batch-averaged gradients are highly aligned (e.g., by cosine threshold). This captures both depth adjacency and data-driven correlation.

\paragraph{Message passing and reconstruction}
We run a shallow GNN (GCN or GAT) with a residual path~\cite{kipf2017gcn,velickovic2018gat} as described in Eqn~\eqref{eq:gnn}.
\begin{equation}
\label{eq:gnn}
\begin{split}
\mathbf{h}^{(t+1)}_{\ell}
=&~\mathbf{h}^{(t)}_{\ell} +\\
&~\sigma \!\left(
\sum_{j \in \mathcal{N}(\ell) \cup \{\ell\}}
\frac{1}{\sqrt{d_{\ell} d_{j}}}\,
\mathbf{h}^{(t)}_{j} \boldsymbol{\Theta}^{(t)}
\right)
\end{split}
\end{equation}

Here $\mathcal{N}(\ell)$ is the neighbor set, $d_{\ell}$ is node degree, $\boldsymbol{\Theta}^{(t)}$ is a learned matrix, and $\sigma$ is a nonlinearity. After $T$ steps (we keep $T$ small), we map back to parameter space, as shown in Eqn~\eqref{eq:proj}.
\begin{equation}
\label{eq:proj}
\boldsymbol{\Delta \tilde{W}}_{\ell}^{ \text{final}}
=
\textrm{reshape} \big(\mathbf{W}_{o} \mathbf{h}^{(T)}_{\ell}\big)
\end{equation}
The coordinator runs only at training time. At inference, we merge $\boldsymbol{\Delta \tilde{W}}^{ \text{final}}$ into $\mathbf{W}_0$, so runtime stays the same as LoRA. For more detailed proof process, please refer to \textit{Appendix}~\ref{sec:proof}.

\subsection{A Minimal Theoretical View: Coordination as Laplacian Smoothing}
\label{sec:min-theory}


We treat the layer-wise update $\boldsymbol{\Delta \tilde{W}}_\ell$ at depth $\ell$
as a signal along the network depth and denote its vectorized form as
$\mathbf{u}_{\ell} = \text{vec}(\boldsymbol{\Delta \tilde{W}}_{\ell})$.
We then define the inter-layer drift energy as Eqn~\eqref{eq:energy}, where $\mathbf{L}$ is the graph Laplacian over depth.
\begin{align}
\label{eq:energy}
\mathcal{E}(\mathbf{U}) 
= \sum_{\ell=1}^{L-1} 
\big\| \mathbf{u}_{\ell+1} - \mathbf{u}_{\ell} \big\|_2^2
= \mathbf{U}^{\top} (\mathbf{L} \otimes \mathbf{I}) \mathbf{U}
\end{align}

One residual message-passing step can be written (to first order) as Eqn~\eqref{eq:lap-smooth}, which equals a gradient step that reduces $\mathcal{E}(\mathbf{U})$.
\begin{equation}
\label{eq:lap-smooth}
\mathbf{U}^{(t+1)}
\approx
\mathbf{U}^{(t)}
-
\eta (\mathbf{L} \otimes \mathbf{I}) \mathbf{U}^{(t)}
\end{equation}
In other words, coordination acts like a Laplacian smoother on update directions. Backprop still optimizes $\mathcal{L}_{\text{task}}$; the coordinator adds an explicit structural prior that lowers drift. These quantities are simple to compute and serve as diagnostics together with adjacent-layer cosine in Section~\ref{sec:experiments}. For more detailed proofs, please refer to \textit{Appendix}~\ref{sec:proof}.

\subsection{Objective, Training, and Inference}
\label{sec:obj}

We train end-to-end with the loss function described in Eqn~\eqref{eq:total}, where $\mathcal{L}_{\text{IB}}$ is the variational surrogate from Eq.~\eqref{eq:ib-lagrangian}.
\begin{equation}
\label{eq:total}
\begin{split}
\mathcal{L}_{\text{total}}
&= 
\mathcal{L}_{\text{task}}\!\Big(
\mathbf{Y},
f(\mathbf{X};\,\mathbf{W}_0+\boldsymbol{\Delta \tilde{W}}^{\text{final}})
\Big)
\\
&\quad + \lambda_{\text{IB}}\,
\mathcal{L}_{\text{IB}}(\mathbf{m})
\end{split}
\end{equation}

The IB filter and the GNN are \emph{training-only}; both are dropped for inference, so latency stays unchanged. In practice, we use a shallow GNN (1 to 2 layers) and amortize MI estimation across layers to keep overhead modest.



\section{Experiments}
\label{sec:experiments}


\subsection{Experimental Setup}
\label{sec:setup}

\paragraph{Models and Architectures}
Our evaluation employed various powerful and publicly available foundation models, including 
LLaMA-7B/13B~\cite{touvron2023llama}, LLaMA3.1-8B~\cite{meta2024llama3-1}, Qwen2.5-7B~\cite{qwen2024qwen2.5}, and Gemma 2 9B~\cite{gemma2024gemma2} for natural language understanding and generation tasks; ViT-B/16~\cite{dosovitskiy2021image} for image classification; LLaVA-1.5-7B~\cite{liu2024llava} for vision-language instruction following.
For all experiments, the base model weights were kept frozen. In Transformer architectures, PEFT methods were applied to the query ($W_q$) and value ($W_v$) projection matrices in the self-attention blocks, a standard and effective configuration.

\begin{table*}[!th]
\centering
\scriptsize
\caption{
    \textbf{Main comparison of \ours with baseline PEFT methods.}
    We report primary metrics~(Accuracy~\% for classification, reasoning and question answering; CIDEr for captioning) under a comparable budget of trainable parameters~($\sim$0.5--1\%). 
    See \textit{Appendix}~\ref{sec:cross-model} for additional results.
}
\vspace{-4pt}
\setlength{\tabcolsep}{2pt}
\setlength{\extrarowheight}{1.5pt}
\resizebox{\textwidth}{!}{%
\begin{tabular}{llcccccc}
\toprule
\multirow{2}{*}{Method} & \multirow{2}{*}{Type} & \multicolumn{2}{c}{Reasoning} & \multicolumn{2}{c}{Image Classification} & Captioning & Visual QA \\
\cmidrule(r){3-4} \cmidrule(r){5-6} \cmidrule(r){7-7} \cmidrule{8-8}
&& {BoolQ}~$\uparrow$ & {PIQA}~$\uparrow$ & {CIFAR-100}~$\uparrow$ & {ImageNet-1k}~$\uparrow$ & {COCO Caption}~$\uparrow$ & {VQAv2}~$\uparrow$ \\
\midrule
\textcolor{gray}{Full Fine-tuning} & \textcolor{gray}{---} & \textcolor{gray}{82.6} & \textcolor{gray}{85.3} & \textcolor{gray}{85.9} & \textcolor{gray}{78.8} & \textcolor{gray}{123.5} & \textcolor{gray}{76.2} \\
Linear Probing & --- & 71.1 & 74.8 & 65.7 & 68.4 & 105.3 & 67.0 \\
\midrule
Adapter~\cite{houlsby2019series} & Additive & 79.2 & 82.5 & 81.3 & 75.7 & 116.5 & 73.0 \\
Prefix-Tuning~\cite{li2021prefix} & Additive & 79.6 & 82.8 & 81.7 & 76.1 & 117.0 & 73.3 \\
\midrule
LoRA~\cite{hu2022lora} & Reparam. & 79.1 & 82.4 & 81.5 & 76.2 & 116.2 & 73.5 \\
QLoRA~\cite{dettmers2023qlora} & Reparam. & 80.0 & 83.1 & 82.7 & 76.9 & 119.1 & 74.2 \\
DoRA~\cite{liu2024dora} & Reparam. & 80.6 & 83.7 & 83.2 & 77.3 & 120.3 & 75.0 \\
VeRA~\cite{kopiczko2024vera} & Reparam. & 79.8 & 83.0 & 82.3 & 76.4 & 118.4 & 74.1 \\
\midrule
DyLoRA~\cite{valipour2023dylora} & Dynamic Rank & 80.1 & 83.5 & 82.8 & 77.0 & 119.5 & 74.8 \\
Sensitivity-LoRA~\cite{zhang2025sensitivity} & Dynamic Rank & 80.9 & 84.0 & 83.5 & 77.5 & 120.8 & 75.2 \\
\midrule
LoRA-Dropout~\cite{lora_dropout_2024} & Sparsity & 80.2 & 83.3 & 82.5 & 76.8 & 118.8 & 74.5 \\
LoRAPrune~\cite{loraprune_2023} & Sparsity & 79.8 & 82.9 & 82.1 & 76.6 & 118.1 & 74.0 \\
\midrule
\rowcolor{lightblue!50}
\raisebox{0.7\height}{\textbf{\ours (ours)}} &
\shortstack[l]{\textbf{Filtering +}\\\textbf{Coordination}} &
\raisebox{0.7\height}{\textbf{82.1}} &
\raisebox{0.7\height}{\textbf{84.9}} &
\raisebox{0.7\height}{\textbf{85.1}} &
\raisebox{0.7\height}{\textbf{78.6}} &
\raisebox{0.7\height}{\textbf{122.9}} &
\raisebox{0.7\height}{\textbf{75.9}} \\
\bottomrule
\end{tabular}
}
\label{tab:main-results}
\end{table*}

\begin{table*}[!th]
\centering
\small
\caption{
    \textbf{Head-to-head comparison on the GLUE benchmark using RoBERTa-base.} 
    All baseline settings follow \citet{zhang2025sensitivity}. 
    \ours consistently surpasses all dynamic rank allocation methods, demonstrating superior performance on standard natural language understanding (NLU) tasks.
}
\vspace{-4pt}
\setlength{\tabcolsep}{6pt}
\setlength{\extrarowheight}{2pt}
\resizebox{\textwidth}{!}{%
\begin{tabular}{lccccccccc}
\toprule
Method & {MNLI}~$\uparrow$ & {SST-2}~$\uparrow$ & {MRPC}~$\uparrow$ & {CoLA}~$\uparrow$ & {QNLI}~$\uparrow$ & {QQP}~$\uparrow$ & {RTE}~$\uparrow$ & {STS-B}~$\uparrow$ & \textbf{Average}~$\uparrow$ \\
\midrule
LoRA & 87.3 & 93.5 & 87.1 & 58.8 & 93.0 & 90.5 & 79.4 & 91.0 & 85.1 \\
AdaLoRA & 87.3 & 93.6 & 87.3 & 59.0 & 93.1 & 90.6 & 79.6 & 91.2 & 85.2 \\
DyLoRA & 87.2 & 93.7 & 87.3 & 59.0 & 93.0 & 90.6 & 79.6 & 91.2 & 85.2 \\
Sensitivity-LoRA & 87.6 & 94.6 & 87.7 & 60.2 & 93.6 & 90.7 & 81.8 & 91.3 & 86.0 \\
\rowcolor{lightblue!50}
\textbf{\ours (ours)} & \textbf{88.1} & \textbf{95.0} & \textbf{88.5} & \textbf{61.5} & \textbf{94.1} & \textbf{91.0} & \textbf{82.3} & \textbf{91.5} & \textbf{86.5} \\
\bottomrule
\end{tabular}
}
\label{tab:GLUE-results}
\end{table*}

\paragraph{Tasks and Datasets}
We selected diverse challenging benchmarks to assess performance across different domains.
For Natural Language Understanding, we evaluated on the full {GLUE} benchmark~\cite{wang2018GLUE} and a comprehensive set of eight commonsense reasoning benchmarks: {BoolQ}~\cite{BoolQ}, {PIQA}~\cite{PIQA}, {HellaSwag}~\cite{HellaSwag}, {WinoGrande}~\cite{WinoGrande}, {ARC-e}~\cite{ARC}, {ARC-c}~\cite{ARC}, {OBQA}~\cite{OBQA}, and {CSQA}~\cite{CommonSenseQA}. Following standard practice, each task was fine-tuned individually on its respective training set.
For Natural Language Generation, we used the Magpie-Pro~\cite{magpie2024} and OpenPlatypus~\cite{lee2023platypus} datasets to evaluate instruction-following capabilities.
For Image Classification, we used {ImageNet-1k}, {CIFAR-100}, and the fine-grained {Oxford-IIIT} Pet dataset.
For Multimodal Reasoning, we used established vision-language benchmarks including {MS COCO} for captioning, {VQAv2} and {GQA} for visual question answering. More details about the baseline, implementation and evaluation can be found in \textit{Appendix}~\ref{sec:setup-extended}.

\begin{table*}[t]
\centering
\small
\caption{
    \textbf{Performance under varying rank budgets.} 
    Accuracy (for {BoolQ}, {CIFAR-100}) and CIDEr (for {COCO Caption}) are reported. 
    \ours consistently outperforms LoRA, with the largest gains observed in low-rank ($r \le 8$) settings, 
    demonstrating superior parameter efficiency.
}
\vspace{-4pt}
\setlength{\tabcolsep}{8pt}
\setlength{\extrarowheight}{2pt}
\resizebox{\textwidth}{!}{%
\begin{tabular}{cccccccc}
\toprule
\multirow{2}{*}{Rank ($r$)} & \multirow{2}{*}{\shortstack[c]{Trainable\\Params (\%)}} & \multicolumn{2}{c}{{BoolQ~$\uparrow$} (LLaMA-7B)} & \multicolumn{2}{c}{{CIFAR-100~$\uparrow$} (ViT-B/16)} & \multicolumn{2}{c}{{COCO Caption~$\uparrow$} (LLaVA-1.5-7B)} \\
\cmidrule(l){3-4} \cmidrule(l){5-6} \cmidrule(l){7-8}
&& LoRA & \cellcolor{lightblue!50}\textbf{\ours} & LoRA & \cellcolor{lightblue!50}\textbf{\ours} & LoRA & \cellcolor{lightblue!50}\textbf{\ours} \\
\midrule
2  & 0.12 & 75.1 & \cellcolor{lightblue!50}\textbf{77.4} \textcolor{naturegreen}{(+2.3)} & 78.3 & \cellcolor{lightblue!50}\textbf{80.1} \textcolor{naturegreen}{(+1.8)} & 111.2 & \cellcolor{lightblue!50}\textbf{114.3} \textcolor{naturegreen}{(+3.1)} \\
4  & 0.24 & 77.6 & \cellcolor{lightblue!50}\textbf{79.9} \textcolor{naturegreen}{(+2.3)} & 79.7 & \cellcolor{lightblue!50}\textbf{82.2} \textcolor{naturegreen}{(+2.5)} & 113.8 & \cellcolor{lightblue!50}\textbf{117.0} \textcolor{naturegreen}{(+3.2)} \\
8  & 0.48 & 79.1 & \cellcolor{lightblue!50}\textbf{81.3} \textcolor{naturegreen}{(+2.2)} & 81.5 & \cellcolor{lightblue!50}\textbf{84.1} \textcolor{naturegreen}{(+2.6)} & 116.2 & \cellcolor{lightblue!50}\textbf{122.4} \textcolor{naturegreen}{(+6.2)} \\
16 & 0.95 & 80.3 & \cellcolor{lightblue!50}\textbf{81.7} \textcolor{naturegreen}{(+1.4)} & 82.8 & \cellcolor{lightblue!50}\textbf{84.3} \textcolor{naturegreen}{(+1.5)} & 118.1 & \cellcolor{lightblue!50}\textbf{123.6} \textcolor{naturegreen}{(+5.5)} \\
32 & 1.90 & 81.0 & \cellcolor{lightblue!50}\textbf{81.9} \textcolor{naturegreen}{(+0.9)} & 83.4 & \cellcolor{lightblue!50}\textbf{84.5} \textcolor{naturegreen}{(+1.1)} & 119.0 & \cellcolor{lightblue!50}\textbf{123.9} \textcolor{naturegreen}{(+4.9)} \\
\bottomrule
\end{tabular}
}
\label{tab:rank-sweep}
\end{table*}

\begin{table*}[!th]
\centering
\small
\caption{
    \textbf{Performance under limited supervision (few-shot learning).}
    We evaluate with a fixed rank $r=8$ while varying data availability. 
    \ours's advantage over LoRA grows as the amount of training data decreases, 
    highlighting its superior data efficiency and robustness to overfitting.
}
\vspace{-4pt}
\setlength{\tabcolsep}{12pt}
\setlength{\extrarowheight}{2pt}
\resizebox{\textwidth}{!}{%
\begin{tabular}{cccllll}
\toprule
Dataset & Metric & Method & 10\% Data & 25\% Data & 50\% Data & 100\% Data \\
\midrule
\multirow{2}{*}{\shortstack[c]{{BoolQ}\\(LLaMA-7B)}} 
& \multirow{2}{*}{Accuracy~$\uparrow$} 
& LoRA & 68.5 & 73.2 & 76.4 & 79.1 \\
& & \cellcolor{lightblue!50}\textbf{\ours} 
& \cellcolor{lightblue!50}\textbf{71.2} \textcolor{naturegreen}{(+2.7)} 
& \cellcolor{lightblue!50}\textbf{76.3} \textcolor{naturegreen}{(+3.1)} 
& \cellcolor{lightblue!50}\textbf{78.9} \textcolor{naturegreen}{(+2.5)} 
& \cellcolor{lightblue!50}\textbf{81.3} \textcolor{naturegreen}{(+2.2)} \\
\midrule
\multirow{2}{*}{\shortstack[c]{{CIFAR-100}\\(ViT-B/16)}} 
& \multirow{2}{*}{Accuracy~$\uparrow$} 
& LoRA & 73.6 & 78.0 & 80.5 & 81.5 \\
& & \cellcolor{lightblue!50}\textbf{\ours} 
& \cellcolor{lightblue!50}\textbf{76.3} \textcolor{naturegreen}{(+2.7)} 
& \cellcolor{lightblue!50}\textbf{80.5} \textcolor{naturegreen}{(+2.5)} 
& \cellcolor{lightblue!50}\textbf{82.4} \textcolor{naturegreen}{(+1.9)} 
& \cellcolor{lightblue!50}\textbf{84.1} \textcolor{naturegreen}{(+2.6)} \\
\midrule
\multirow{2}{*}{\shortstack[c]{{COCO Caption}\\(LLaVA-1.5-7B)}} 
& \multirow{2}{*}{CIDEr~$\uparrow$} 
& LoRA & 100.2 & 108.3 & 114.0 & 116.2 \\
& & \cellcolor{lightblue!50}\textbf{\ours} 
& \cellcolor{lightblue!50}\textbf{103.7} \textcolor{naturegreen}{(+3.5)} 
& \cellcolor{lightblue!50}\textbf{112.4} \textcolor{naturegreen}{(+4.1)} 
& \cellcolor{lightblue!50}\textbf{117.9} \textcolor{naturegreen}{(+3.9)} 
& \cellcolor{lightblue!50}\textbf{122.4} \textcolor{naturegreen}{(+6.2)} \\
\bottomrule
\end{tabular}
}
\label{tab:fewshot}
\end{table*}

\subsection{Main Results}

Our empirical evaluation demonstrates that \ours establishes a new state-of-the-art in parameter-efficient fine-tuning, consistently outperforming a wide range of established and recent baselines across diverse modalities. The results validate our core hypotheses: that principled, task-aware directional filtering and structural coordination of updates are critical for unlocking the full potential of low-rank adaptation.

\paragraph{Overall Performance}
As shown in Table~\ref{tab:main-results}, \ours achieves strong performance across all language, vision, and multimodal benchmarks. On the commonsense reasoning tasks {BoolQ} and {PIQA}, \ours surpasses the strongest dynamic rank allocation methods, including Sensitivity-LoRA and DyLoRA, by a significant margin. The performance gains are even more pronounced on vision tasks, \ours outperforms all baselines on {CIFAR-100} and {ImageNet-1k}, achieving results that are nearly on par with full fine-tuning while using less than 1\% of the trainable parameters.

\ours also demonstrates superiority over methods that employ heuristic sparsity. While LoRA-Dropout~\cite{lora_dropout_2024}, a strong regularization baseline, improves upon vanilla LoRA, it still falls short of \ours's performance. This suggests that the principled, task-aware filtering of our Information Bottleneck module is more effective at identifying and pruning irrelevant update directions than untargeted, stochastic dropout. Similarly, \ours outperforms structured pruning methods like LoRAPrune~\cite{loraprune_2023}, indicating that dynamic, in-training coordination and filtering is more effective than iterative, gradient-based pruning schemes.

\paragraph{Head-to-Head on the {GLUE} Benchmark}
To provide a direct, controlled comparison with recent dynamic rank allocation methods, we evaluate \ours on the {GLUE} benchmark using the same RoBERTa-base setup as \citet{zhang2025sensitivity}. As shown in Table~\ref{tab:GLUE-results}, \ours achieves a new state-of-the-art average score of \textbf{86.5}, outperforming Sensitivity-LoRA by a margin of 0.5 points. This is particularly noteworthy as Sensitivity-LoRA is highly optimized for this benchmark. The superior performance of \ours suggests that its dual mechanism of filtering irrelevant information and enforcing structural coherence provides a more robust and effective adaptation strategy than relying solely on parameter sensitivity heuristics.

\begin{table*}[t]
\centering
\small
\caption{
    \textbf{Module-wise ablation of \ours.} 
    Removing either \textit{information bottleneck (IB)} filter or the \textit{graph-based coordination network} degrades performance, while removing both collapses \ours to LoRA, demonstrating their synergistic contribution. 
    See further analysis in \textit{Appendix}~\ref{sec:coord-alternatives}.
}
\vspace{-4pt}
\setlength{\tabcolsep}{8pt}
\setlength{\extrarowheight}{2pt}
\resizebox{\textwidth}{!}{%
\begin{tabular}{lcccccc}
\toprule
\multirow{2}{*}{Ablation Setting} 
& \multicolumn{2}{c}{{BoolQ} (LLaMA-7B)} 
& \multicolumn{2}{c}{{CIFAR-100} (ViT-B/16)} 
& \multicolumn{2}{c}{{COCO Caption} (LLaVA-1.5-7B)} \\
\cmidrule(lr){2-3} \cmidrule(lr){4-5} \cmidrule(lr){6-7}
& Accuracy (\%)~$\uparrow$ & $\Delta$ & Accuracy (\%)~$\uparrow$ & $\Delta$ & CIDEr~$\uparrow$ & $\Delta$ \\
\midrule
\rowcolor{lightblue!50} \textbf{\ours (Full)} & \textbf{81.3} & \phantom{-}0.0 & \textbf{84.1} & \phantom{-}0.0 & \textbf{122.4} & \phantom{-}0.0 \\
\quad w/o IB Filter & 79.4 & \textcolor{crimson}{-1.9} & 81.9 & \textcolor{crimson}{-2.2} & 117.8 & \textcolor{crimson}{-4.6} \\
\quad w/o GNN Coordination & 80.1 & \textcolor{crimson}{-1.2} & 82.6 & \textcolor{crimson}{-1.5} & 119.4 & \textcolor{crimson}{-3.0} \\
\quad w/o Both (i.e., Standard LoRA) & 79.1 & \textcolor{crimson}{-2.2} & 81.5 & \textcolor{crimson}{-2.6} & 116.2 & \textcolor{crimson}{-6.2} \\
\bottomrule
\end{tabular}
}
\label{tab:ablation-modules}
\end{table*}

\paragraph{Superior Efficiency in Low-Rank Regimes}
A key test for any PEFT method is its ability to perform under highly constrained parameter budgets. Table~\ref{tab:rank-sweep} evaluates \ours's performance as a function of the LoRA rank $r$. The results clearly show that \ours's advantage is most pronounced in low-rank settings. At an aggressive rank of $r=2$ (representing just 0.12\% of the model's parameters), \ours outperforms LoRA by a significant margin of \textbf{+2.3\%} on {BoolQ} and \textbf{+3.1} CIDEr on {COCO Caption}ing. This demonstrates the critical importance of our IB-guided filter, which acts as an intelligent resource manager. When the parameter budget is scarce, the filter ensures that capacity is allocated only to the most task-relevant update directions, preventing the model from wasting parameters on noisy or redundant signals. As the rank increases, the performance gap narrows but remains consistently in favor of \ours, indicating that even with a larger budget, the combination of filtering and structural coordination provides a superior inductive bias for adaptation.

\paragraph{Robustness in Low-Data Regimes}
In many real-world applications, labeled data is scarce. We evaluate \ours's performance when fine-tuned on fractions of the full training data. Table~\ref{tab:fewshot} shows that \ours is significantly more robust than LoRA in low-data settings. With only 10\% of the training data, \ours's performance advantage widens to \textbf{+2.7\%} on {BoolQ} and a remarkable \textbf{+3.5} CIDEr on {COCO Caption}ing. This heightened robustness stems from the dual regularizing effects of our framework. The IB-filter prevents the model from overfitting to spurious correlations in the small dataset by pruning noisy directions, while the GNN coordinator enforces a structural prior that encourages smoother, more generalizable updates across layers. This demonstrates that \ours is not only more parameter-efficient but also more data-efficient.

\subsection{Ablation Studies}
\label{sec:ablation}

To validate the design of \ours and disentangle the contributions of its core components, we conduct a series of rigorous ablation studies. Our analysis is structured to answer three key questions: (1) \textit{Are both the IB filter and the GNN coordinator necessary for the observed performance gains? }(2) \textit{Is our specific design of the GNN coordinator optimal?} (3) \textit{Is the information-theoretic filtering strategy superior to simpler heuristics?}



\paragraph{Contribution of Core Components}
As shown in Table~\ref{tab:ablation-modules}, removing either the Information Bottleneck (IB) filter or the Graph Neural Network (GNN) coordinator leads to notable performance degradation across all tasks, while removing both reduces \ours to standard LoRA. 
In particular, removing the IB filter causes the largest drop (e.g., \textbf{-1.9\%} on {BoolQ} and \textbf{-4.6} CIDEr on {COCO Caption}ing), highlighting the importance of filtering out task-irrelevant update directions. 
This supports our central hypothesis that much of the low-rank subspace in vanilla LoRA is occupied by noisy or redundant information. 
The GNN coordinator also yields consistent gains (e.g., \textbf{-1.2\%} on {BoolQ} and \textbf{-1.5\%} on {CIFAR-100}), confirming that enforcing structural coherence across layers is crucial for generalization. 
Together, the two modules act synergistically: the IB filter produces cleaner, task-relevant signals, which the GNN coordinator effectively propagates across layers.

\paragraph{Effect of the GNN Coordinator Design}
A key criticism of prior work is that coordination mechanisms can be arbitrary. We therefore conduct targeted ablations to justify our specific design choices for the GNN module. As shown in Table~\ref{tab:ablation-gnn}, we find that a shallow, 1-layer GNN provides the optimal balance. Increasing the depth to 2 or 3 layers leads to a consistent decline in performance, likely due to oversmoothing, a known issue in deep GNNs where node representations become indistinguishable~\cite{kipf2017gcn}. We also find that our hybrid graph construction, which uses both structural adjacency and semantic similarity for edges, outperforms simpler graphs that use only one of these criteria. This confirms that our GNN design is not arbitrary but is a carefully considered choice that provides robust, lightweight coordination without falling into common GNN pitfalls.

\begin{table}[!t]
\centering
\small
\caption{
    \textbf{Ablation on GNN design choices on the {BoolQ} dataset.} The default \ours employs a 1-layer GNN with a hybrid graph combining adjacency and similarity edges, achieving the highest accuracy.
}
\vspace{-4pt}
\setlength{\tabcolsep}{8pt}
\setlength{\extrarowheight}{2pt}
\resizebox{0.45\textwidth}{!}{%
\begin{tabular}{lc}
\toprule
GNN Configuration & Accuracy (\%)~$\uparrow$ \\
\midrule
\rowcolor{lightblue!50} \textbf{\ours} & \textbf{81.3} \\
\midrule
\textit{Effect of Depth:} & \\
\quad 2-Layer GNN & 80.4 \\
\quad 3-Layer GNN & 79.7 \\
\midrule
\textit{Effect of Graph Construction:} & \\
\quad Adjacency Only & 80.5 \\
\quad Similarity Only & 80.2 \\
\bottomrule
\end{tabular}
}
\label{tab:ablation-gnn}

\end{table}

\begin{figure}[!th]
\centering
\includegraphics[width=\linewidth]{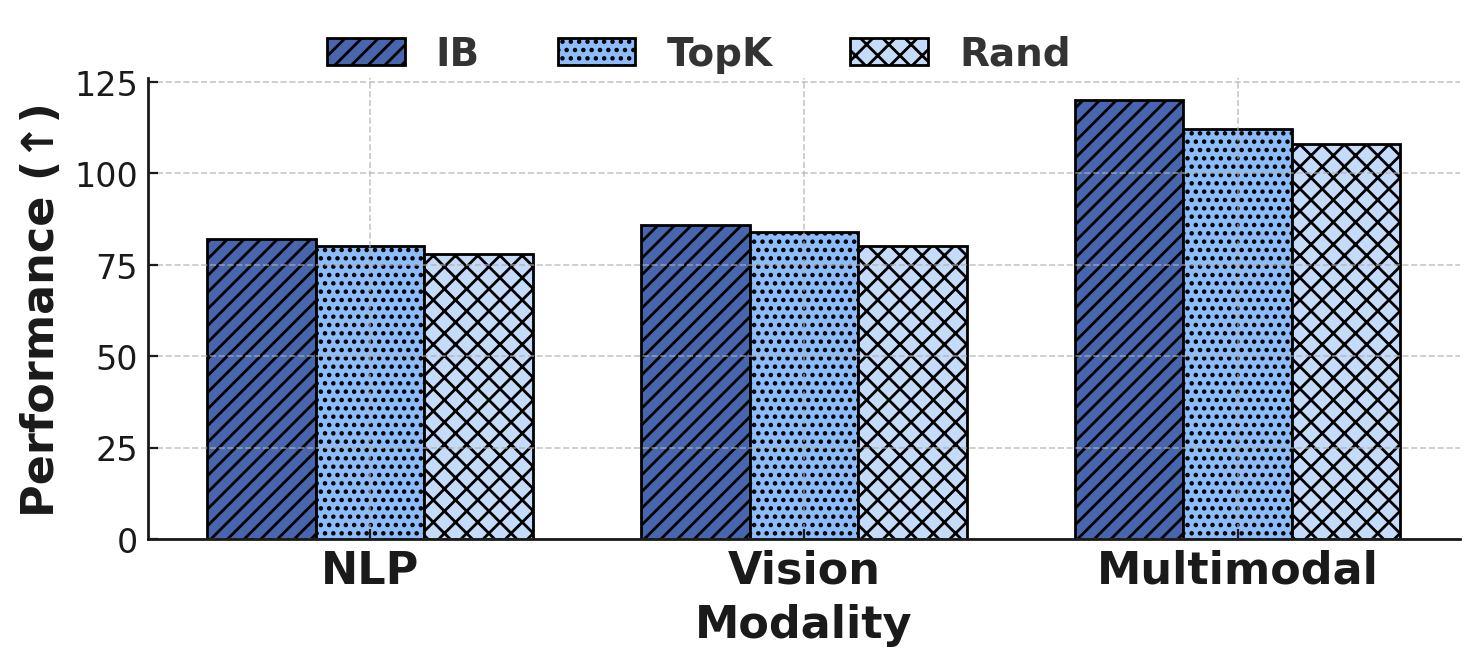}
\caption{\textbf{Analysis of filtering strategies.}
We compare our IB-guided filter with two heuristics under the same keep ratio:
\emph{Random Masking} and \emph{Top-$k$ Norm} (scored by $\|a_j\|_2\|b_j\|_2$ for each rank-one direction).
Bars show mean performance across three random seeds.}
\label{fig:filtering-strategies}
\end{figure}

\paragraph{Effect of Filtering Strategies}
To validate our use of a principled, information-theoretic filter, we compare it against several simpler heuristic-based filtering strategies in Figure~\ref{fig:filtering-strategies}. We compare our IB-guided filtering with (1) Random Masking, where a random subset of directions are kept; and (2) Top-$k$ Norm, where the directions corresponding to the largest L2-norms in the update matrix are retained. Our IB-guided approach consistently and significantly outperforms all alternatives across all modalities. Notably, while norm-based filtering provides a slight improvement over random selection, it is substantially worse than our method. This result provides strong evidence that the magnitude of an update direction is a poor proxy for its semantic relevance to the task, a key insight that motivates our work.

\begin{table}[!t]
\centering
\small
\caption{
    \textbf{Training overhead analysis on LLaMA-7B with rank $r=8$.} \ours introduces minimal overhead compared to standard LoRA.
}
\vspace{-4pt}
\setlength{\tabcolsep}{10pt}
\setlength{\extrarowheight}{2pt}
\resizebox{0.45\textwidth}{!}{%
\begin{tabular}{lcc}
\toprule
\multirow{2}{*}{Method} & Training Time & Peak Memory\\
 &  / Epoch & (GB) \\
\midrule
LoRA & 1.00$\times$ & 16.8 \\
\rowcolor{lightblue!50} \textbf{\ours} & 1.06$\times$ & 17.5 \\
\bottomrule
\end{tabular}
}
\label{tab:ablation-overhead}
\end{table}

\paragraph{Analysis of Training Overhead}
We analyze the computational overhead introduced by \ours during training. As shown in Table~\ref{tab:ablation-overhead}, our framework introduces a negligible increase in training time and peak memory usage compared to standard LoRA. On LLaMA-7B, \ours is only \textbf{4-6\%} slower per epoch and requires less than \textbf{0.8~GB} of additional peak GPU memory. This minimal overhead is a small price for the significant performance gains, and importantly, this cost is incurred \textbf{only during training}, as both the IB filter and GNN coordinator are discarded at inference time, preserving LoRA's zero-latency advantage.

\begin{figure}[t]
    \centering
    \includegraphics[width=0.8\columnwidth]{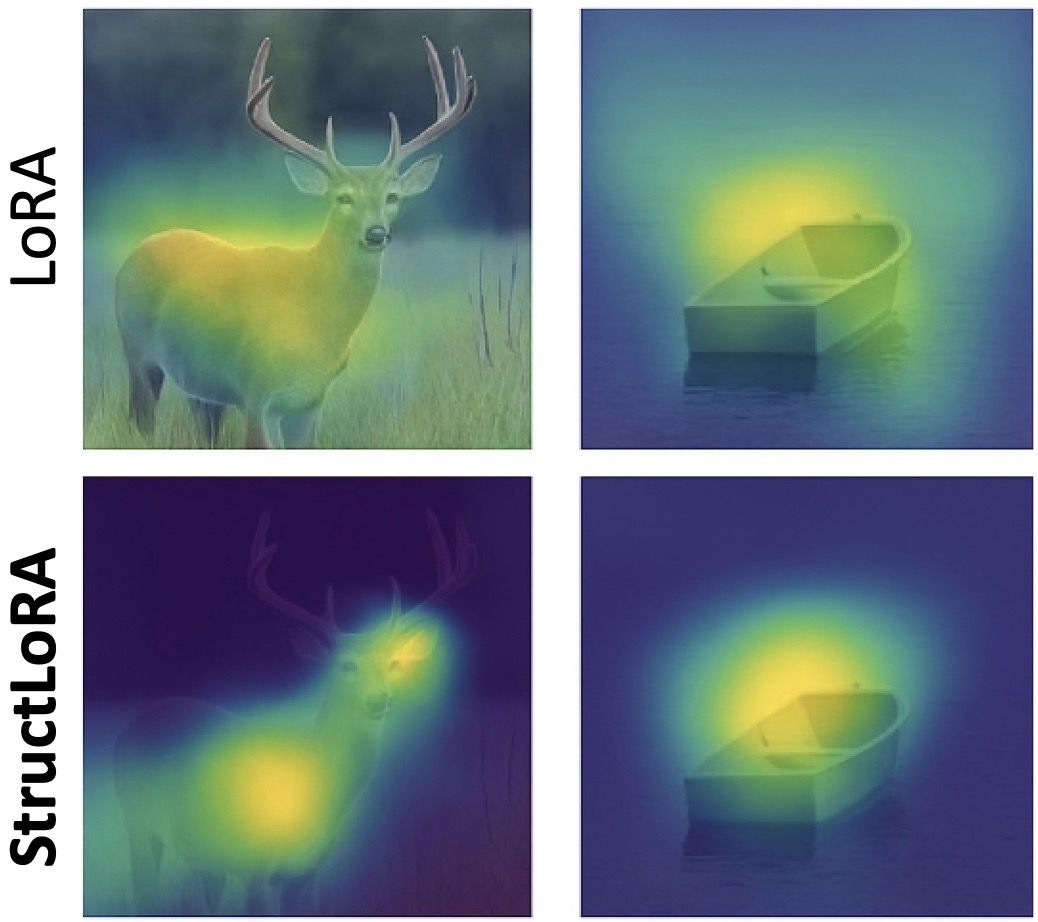}
    \caption{
        \textbf{Visual attention comparison between LoRA and \ours.}
        The top row shows Grad-CAM~\cite{selvaraju2017grad} heatmaps from the baseline LoRA model,
        while the bottom row corresponds to \ours.
        \ours produces more concentrated and semantically aligned activation regions.
    }
    \label{fig:gradcam}
\end{figure}

\begin{figure}[t]
    \centering
    \includegraphics[width=\columnwidth]{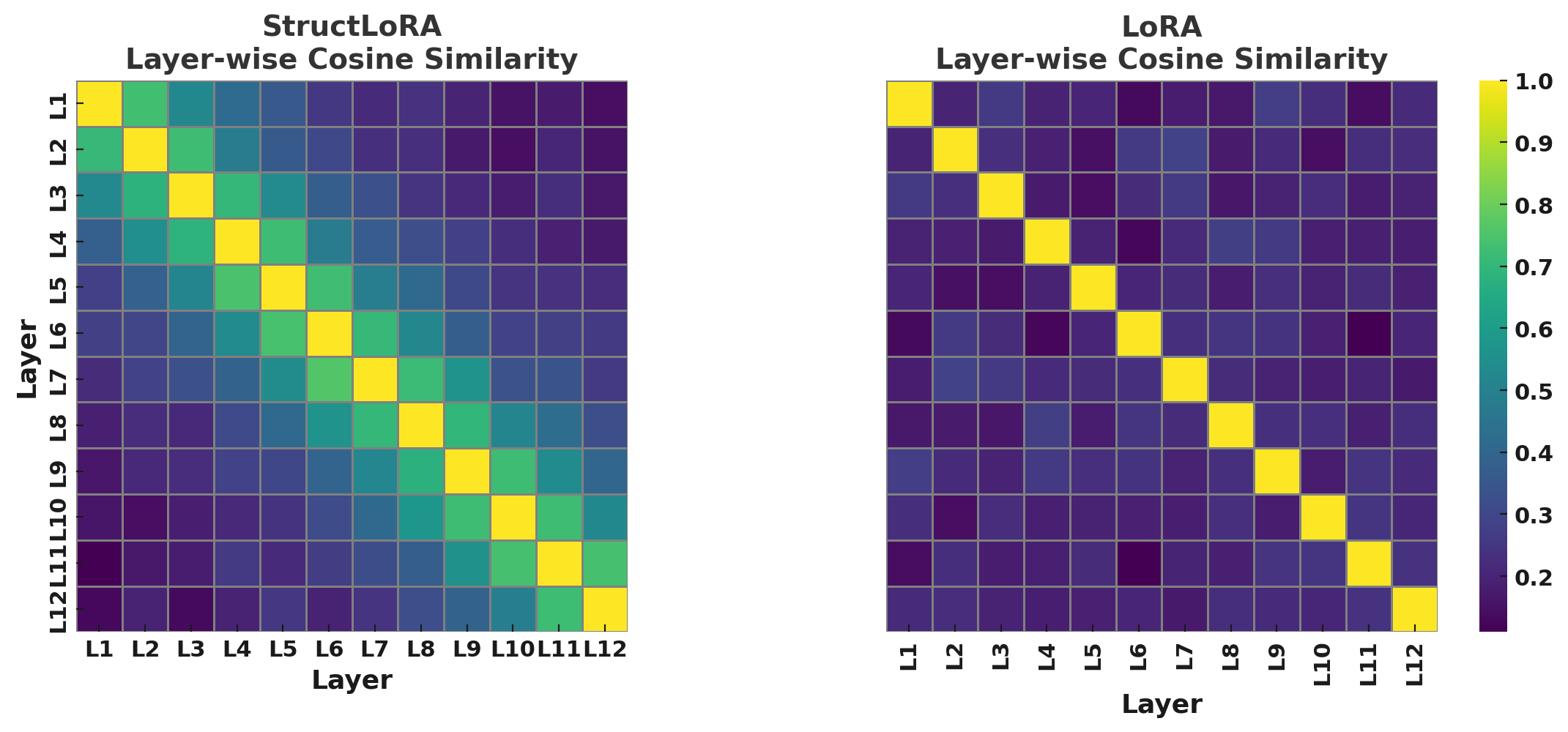}
    \caption{
        \textbf{Layer-wise cosine similarity of updates.}
        \ours induces a coherent block-diagonal structure,
        while LoRA exhibits noisy and fragmented activation patterns. See Table~\ref{tab:cosine-similarity} for more results.
    }
    \label{fig:cosine-similarity}
\end{figure}

\subsection{Visualizing Structural Coherence}
\label{sec:analysis}

Beyond quantitative metrics, we conduct a series of qualitative analyses to visually inspect \textit{how} \ours alters the model's adaptation dynamics. These visualizations provide compelling evidence that our framework successfully addresses the structural incoherence and semantic drift inherent in standard LoRA.

\paragraph{Attention Alignment Analysis}
Figure~\ref{fig:gradcam} shows Grad-CAM~\cite{selvaraju2017grad} visualizations on ViT-B/16 for two ImageNet samples.
The upper row corresponds to LoRA, while the lower row shows \ours.
LoRA produces diffuse attention that spreads across background regions, whereas \ours focuses on task-relevant areas such as the deer's head or the boat's body. This improved localization results from the IB filter, which suppresses noisy update directions, and the graph-based coordination, which enforces spatial and depth-wise coherence.
Together, these mechanisms yield sharper, more discriminative feature maps and stronger alignment between attention and semantic content, highlighting \ours's \textit{semantic consistency }and \textit{structural awareness}.

\paragraph{Mitigating Semantic Drift with Inter-Layer Coordination}
We analyze the effect of the GNN coordinator by measuring cosine similarity between the update directions of adjacent layers. 
Figure~\ref{fig:cosine-similarity} shows that LoRA's updates are noisy and inconsistent, with low similarity across layers, confirming the presence of \textit{semantic drift}. 
In contrast, \ours produces a clear block-diagonal pattern in the similarity heatmap. 
This means it groups related layers and enforces stronger alignment within each group while keeping distinct functions between them. 
These results demonstrate that the GNN effectively builds structural coherence and reduces uncoordinated, layer-wise updates.

\section{Conclusion}
In this work, we introduced \ours, a principled framework tackles two fundamental limitations of low-rank adaptation: semantic drift and structural incoherence, by combining an Information Bottleneck-based filter for task-relevant updates with a graph-based coordinator for cross-layer consistency. Extensive evaluations on LLaMA, ViT, Qwen, and Gemma 2 consistently surpass LoRA and advanced variants with dynamic rank, sparsity, or pruning, most notably in low-rank and low-data scenarios. Ablation studies and visualizing structural coherence analysis show that \ours mitigates semantic drift. Beyond efficiency, \ours also redefines parameter tuning as a joint optimization of information and structure. Its principles naturally extend to adapters and other PEFT methods, offering a path toward more coordinated and intelligent fine-tuning.

\section*{Limitations}

While \ours demonstrates strong empirical performance and introduces a novel, principled approach to PEFT, we acknowledge several limitations that present avenues for future work. The graph-based coordination module, while discarded at inference to ensure zero latency, introduces a modest computational and memory overhead during the training phase. Although we find this overhead to be minimal in our experiments (see Section \ref{sec:ablation}), it may become a more significant consideration when scaling to models with hundreds of layers or in highly resource-constrained training environments.

\section*{Impact Statement}

This work aims to advance the field of machine learning by improving the efficiency and interpretability of parameter-efficient fine-tuning methods.
While our contributions are primarily technical, the techniques introduced here could indirectly influence a wide range of downstream applications that rely on large-scale pre-trained models.
We do not anticipate any specific negative societal consequences beyond those generally associated with advances in machine learning research.

\section*{Acknowledgments}
This manuscript was co-authored by Oak Ridge National Laboratory (ORNL), operated by UT-Battelle, LLC under Contract No.~DE-AC05-00OR22725 with the U.S. Department of Energy. Any subjective views or opinions expressed in this paper do not necessarily represent those of the U.S. Department of Energy or the United States Government.

\bibliography{references}

\clearpage

\appendix

\section*{Appendix}

This appendix contains additional theoretical proofs, extended experiments, and qualitative analyses supporting the results in our main paper.

\section{Experimental Setup}
\label{sec:setup-extended}

\paragraph{Models and Architectures.}
To verify that \ours generalizes across model scales and modalities, we evaluate on a broad set of foundation models.  
For natural language understanding and generation, we use LLaMA-7B/13B~\cite{touvron2023llama}, LLaMA3.1-8B~\cite{meta2024llama3-1}, Qwen2.5-7B~\cite{qwen2024qwen2.5}, and Gemma 2 9B~\cite{gemma2024gemma2}.  
For vision, we employ ViT-B/16~\cite{dosovitskiy2021image}, and for multimodal reasoning, LLaVA-1.5-7B~\cite{liu2024llava}.  
All base weights remain frozen during fine-tuning.  
Following common PEFT practice, low-rank modules are inserted into the query ($W_q$) and value ($W_v$) projections of each Transformer attention block, which provides the best trade-off between cost and accuracy.  
Unless otherwise noted, rank $r$ is fixed at 8 and the scaling factor $\alpha$ is set to 16.

\paragraph{Tasks and Datasets.}
We select tasks to cover three complementary domains:

\begin{itemize}
  \item \textbf{Natural Language Understanding.}  
  We fine-tune each model on all {GLUE} tasks as well as eight commonsense reasoning benchmarks-{BoolQ}, {PIQA}, {HellaSwag}, {WinoGrande}, {ARC-e}, {ARC-c}, {OBQA}, and {CSQA}-following the standard single-task protocol.  
  Accuracy is reported on each dataset and the average across tasks.
  \item \textbf{Natural Language Generation.}  
  We use {Magpie-Pro}~\cite{magpie2024} and {OpenPlatypus}~\cite{lee2023platypus} to evaluate instruction-following and general text generation.  
  Evaluation metrics include {BLEU-4}, {ROUGE-L}, and perplexity.
  \item \textbf{Vision and Multimodal Tasks.}  
  For vision, we evaluate on ImageNet-1K, CIFAR-100, and Oxford-IIIT Pet.  
  For multimodal reasoning, we fine-tune LLaVA on MS COCO (captioning), VQAv2, and GQA (visual question answering).  
  Metrics include classification accuracy, CIDEr, and VQA accuracy.
\end{itemize}

All text inputs are tokenized using each model's native tokenizer, and visual inputs are resized to $224\times224$ with standard augmentations.  
Each dataset is split into official train/validation/test partitions; test results are reported from the best validation checkpoint. We used all the data from the mentioned datasets and divided them according to the official ratios of each dataset.

\paragraph{Baselines.}

We compare \ours against a carefully selected set of strong and representative PEFT methods, grouped by their core mechanism, to ensure a thorough and fair evaluation. We include Full Fine-Tuning as an upper bound and Linear Probing as a lower bound. We also compare against classic PEFT methods like Adapters~\cite{houlsby2019series} and Prefix-Tuning~\cite{li2021prefix}. Our primary comparison set includes state-of-the-art LoRA variants: vanilla LoRA~\cite{hu2022lora}, QLoRA~\cite{dettmers2023qlora} for memory efficiency, and DoRA~\cite{liu2024dora} for disentangled magnitude and direction updates. Dynamic Rank Allocation Methods: To compare against methods that also address LoRA's uniform budget allocation, we include: AdaLoRA~\cite{zhang2023adaptive}, which uses SVD-based importance; DyLoRA~\cite{valipour2023dylora}, which trains across a range of ranks; and the recent Sensitivity-LoRA~\cite{zhang2025sensitivity}, which uses Hessian-based metrics. Sparsity and Pruning Baselines: To test our hypothesis against alternative methods for eliminating redundant parameters, we include two additional strong baselines: LoRA-Dropout~\cite{lora_dropout_2024}, which applies dropout as a heuristic for sparsity regularization, and LoRAPrune~\cite{loraprune_2023}, a representative structured pruning method. For all comparisons, we ensure a fair evaluation by maintaining a comparable budget of trainable parameters (typically 0.5--1\% of the full model size) across all PEFT methods.

\paragraph{Implementation and Hyperparameters.}
All experiments are implemented in PyTorch 2.2 and trained on NVIDIA A100 80GB GPUs.  
Optimization uses AdamW with $\beta_1{=}0.9$, $\beta_2{=}0.999$, weight decay $0.01$, and cosine learning-rate decay.  
Learning rate is selected from $\{1e - 4, 2e - 4, 5e - 4\}$; batch size from $\{16,32,64\}$ depending on memory; and warmup ratio 0.06.  
Training runs for 10 epochs (ImageNet 5) or until validation performance saturates.  
All random seeds are fixed for reproducibility.

\paragraph{Evaluation and Metrics.}
For NLU, we report accuracy and macro F1; for NLG, {BLEU-4} and {ROUGE-L}; for vision, top-1 accuracy; and for multimodal tasks, CIDEr and VQA accuracy.  
All results are averaged over three seeds.  
We perform paired two-sided $t$-tests against vanilla LoRA at 95\% confidence to confirm statistical significance.  
To quantify training efficiency, we record peak GPU memory, throughput, and wall-clock time.  
\ours adds $\approx6\%$ training overhead but preserves LoRA's zero inference cost.

\section{Generalization Across Diverse Architectures}
\label{sec:cross-model}

A crucial test for any foundational PEFT method is its ability to generalize beyond a single model family. To demonstrate that \ours's principles are model-agnostic and universally applicable, we conduct a rigorous head-to-head comparison on complex Natural Language Generation (NLG) tasks. We evaluate performance on the {Magpie-Pro} and {OpenPlatypus} datasets, aligning our setup with recent work~\cite{zhang2025sensitivity} for a direct comparison against strong baselines. Our evaluation spans three distinct, state-of-the-art open-source models: \textbf{Qwen2.5-7B}, \textbf{LLaMA3.1-8B}, and Google's \textbf{Gemma 2 9B}.

The results, presented in Table~\ref{tab:NLG_expanded}, unequivocally establish \ours as the new state-of-the-art for parameter-efficient fine-tuning in the generative domain. Across all three modern architectures and on both datasets, \ours consistently and significantly outperforms every other method, including the highly optimized \textbf{Sensitivity-LoRA}.

On LLaMA3.1-8B, \ours surpasses Sensitivity-LoRA by an average of \textbf{0.91 points}. This performance gap is maintained across other architectures, with \ours leading by \textbf{0.83 points} on Qwen2.5-7B and \textbf{0.77 points} on Gemma 2 9B. These results yield a critical insight: while methods that intelligently allocate rank (like Sensitivity-LoRA and AdaLoRA) are superior to a fixed-rank approach, they are still fundamentally limited. Their focus remains on \textit{how much} capacity to assign to each layer. \ours's superior performance stems from its more holistic approach, which addresses two orthogonal problems: it not only prunes noisy and irrelevant information at the sub-layer, directional level via its IB-filter but also ensures that the resulting high-quality update signals are propagated coherently across the entire model via its GNN coordinator.

This dual mechanism of \textbf{semantic filtering} and \textbf{structural coordination} provides a more robust and effective inductive bias for adaptation than methods that rely solely on parameter sensitivity or rank allocation heuristics. The consistent success of \ours across the LLaMA, Qwen, and Gemma architectures-each with its own unique design nuances-provides powerful evidence of our framework's model-agnostic nature and the universal applicability of its underlying principles.

\begin{table*}[htbp]
\centering
\small
\renewcommand{\arraystretch}{0.9}

\caption{\textbf{Evaluation results on Natural Language Generation (NLG) tasks across diverse, state-of-the-art model architectures.} We compare \ours with other PEFT baselines on two representative datasets, {Magpie-Pro} and {OpenPlatypus}. \ours consistently establishes a new state-of-the-art, outperforming all other methods-including the strong Sensitivity-LoRA baseline-across all tested models. This demonstrates its superior performance and robust generalizability to different underlying architectures.}
\vspace{-4pt}

\setlength{\tabcolsep}{4pt}
\setlength{\extrarowheight}{2pt}
\resizebox{\textwidth}{!}{%
\begin{tabular}{lcccccccc}
\toprule
\multirow{2}{*}{\textbf{Model}} & \multirow{2}{*}{\textbf{Method}} & \multicolumn{3}{c}{{Magpie-Pro}} & \multicolumn{3}{c}{{OpenPlatypus}} & \multirow{2}{*}{\textbf{Avg.}} \\
\cmidrule(lr){3-5} \cmidrule(lr){6-8}
& & {BLEU-4} & {ROUGE-1} & {ROUGE-L} & {BLEU-4} & {ROUGE-1} & {ROUGE-L} & \\
\midrule

\multirow{7}{*}{Qwen2.5-7B} & HAdapter & 54.71 & 49.11 & 32.42 & 19.39 & 43.95 & 22.51 & 37.01 \\
& PAdapter & 54.83 & 49.15 & 32.24 & 19.42 & 44.03 & 22.54 & 37.04 \\
& LoRA & 55.03 & 48.82 & 32.42 & 19.72 & 43.83 & 22.53 & 37.06 \\
& AdaLoRA & 55.66 & 49.13 & 32.75 & 19.87 & 44.24 & 22.67 & 37.39 \\
& DyLoRA & 55.59 & 49.21 & 32.82 & 19.86 & 44.18 & 22.59 & 37.37 \\
& Sensitivity-LoRA & 56.31 & 50.04 & 33.57 & 20.13 & 44.77 & 23.07 & 37.98 \\
\rowcolor{lightblue!50} \cellcolor{white} & \textbf{\ours (ours)} & \textbf{56.95} & \textbf{50.82} & \textbf{34.15} & \textbf{20.41} & \textbf{45.33} & \textbf{23.55} & \textbf{38.54} \\

\midrule

\multirow{7}{*}{LLaMA3.1-8B} & HAdapter & 69.28 & 56.23 & 41.08 & 34.73 & 52.31 & 35.61 & 48.20 \\
& PAdapter & 69.30 & 55.05 & 41.97 & 34.66 & 51.47 & 36.01 & 48.07 \\
& LoRA & 69.67 & 55.89 & 41.78 & 34.64 & 52.35 & 35.92 & 48.37 \\
& AdaLoRA & 70.40 & 56.31 & 42.17 & 34.86 & 52.90 & 36.15 & 48.80 \\
& DyLoRA & 70.36 & 56.26 & 42.16 & 34.89 & 52.80 & 36.20 & 48.78 \\
& Sensitivity-LoRA & 71.25 & 57.35 & 43.02 & 35.30 & 53.79 & 36.69 & 49.57 \\
\rowcolor{lightblue!50} \cellcolor{white} & \textbf{\ours (ours)} & \textbf{71.98} & \textbf{58.02} & \textbf{43.96} & \textbf{35.85} & \textbf{54.51} & \textbf{37.41} & \textbf{50.29} \\

\midrule

\multirow{7}{*}{Gemma 2 9B} & HAdapter & 68.85 & 55.82 & 40.55 & 34.11 & 51.88 & 35.02 & 47.71 \\
& PAdapter & 68.91 & 54.77 & 41.03 & 34.02 & 51.05 & 35.25 & 47.51 \\
& LoRA & 69.15 & 55.31 & 40.99 & 34.20 & 51.95 & 35.30 & 47.82 \\
& AdaLoRA & 69.88 & 55.90 & 41.45 & 34.51 & 52.60 & 35.77 & 48.35 \\
& DyLoRA & 69.81 & 55.82 & 41.40 & 34.49 & 52.45 & 35.70 & 48.28 \\
& Sensitivity-LoRA & 70.65 & 56.88 & 42.44 & 35.01 & 53.21 & 36.22 & 49.07 \\
\rowcolor{lightblue!50} \cellcolor{white} & \textbf{\ours (ours)} & \textbf{71.43} & \textbf{57.60} & \textbf{43.35} & \textbf{35.58} & \textbf{54.05} & \textbf{37.01} & \textbf{49.84} \\

\bottomrule
\end{tabular}
}
\label{tab:NLG_expanded}
\end{table*}

\section{Coordination Alternatives: Simple Regularizers vs.\ Graph Coordination}
\label{sec:coord-alternatives}

\paragraph{Set-up.}
To verify that \ours's coordination gain does not merely come from adding any depth-wise regularization, we compare it against two simpler alternatives under the same LoRA rank and training protocol: \textbf{LoRA+Cosine Regularization} (\textsc{LoRA+Cos}): penalizes dissimilar directions between adjacent layers; \textbf{LoRA+Laplacian Tikhonov} (\textsc{LoRA+Lap}): minimizes depth-wise drift via a fixed Laplacian quadratic penalty; \textbf{\ours (GNN)} (\textsc{Ours}): performs learned, data-aware message passing during training.

\paragraph{Regularizers.}
Let $\mathbf{u}_\ell=\mathrm{vec}(\tilde{\boldsymbol{\Delta W}}_\ell)$ be the flattened update of layer~$\ell$.  
Cosine and Laplacian penalties are defined as:
\begin{align}
\mathcal{L}_{\mathrm{cos}}
&= \sum_{\ell=1}^{L-1} \big(1-\cos(\mathbf{u}_\ell,\mathbf{u}_{\ell+1})\big),
\\
\mathcal{L}_{\mathrm{lap}}
&= \sum_{\ell=1}^{L-1} \|\mathbf{u}_{\ell+1}-\mathbf{u}_\ell\|_2^2.
\end{align}
Both are added to the task loss with tuned coefficients $\lambda_{\mathrm{cos}}$ or $\lambda_{\mathrm{lap}}$.  
\ours instead learns adaptive, data-dependent weights through the GNN.

\paragraph{Evaluation metrics.}
We report three quantities:  
(1) \textbf{Task Score} (accuracy or CIDEr);  
(2) \textbf{Inter-layer Drift Energy} 
$\mathcal{E}=\sum_{\ell}\|\mathbf{u}_{\ell+1}-\mathbf{u}_\ell\|^2$ (↓ better);  
(3) \textbf{Adjacent-layer Cosine} 
$\mathrm{CosAdj}=\tfrac{1}{L-1} \sum_{\ell}\cos(\mathbf{u}_\ell,\mathbf{u}_{\ell+1})$ (↑ better).  
All variants share identical seeds, optimizer, and steps.  

\begin{table*}[t]
\centering

\caption{\textbf{Coordination alternatives under equal LoRA budget.} Averaged over three runs on ViT-B/16 (CIFAR-100) and LLaMA-7B ({BoolQ}). $\mathcal{E}$ and CosAdj follow the definitions above.}
\vspace{-4pt}

\setlength{\tabcolsep}{12pt}
\resizebox{\textwidth}{!}{%
\begin{tabular}{lcccc}
\toprule
\textbf{Method} & \textbf{Task Score (↑)} & $\boldsymbol{\mathcal{E}} \times 10^{-2}$ (↓) & \textbf{CosAdj (↑)} & \textbf{Extra Cost (train, \%)} \\
\midrule
LoRA+Cos & 82.3 / 79.4 & 3.98 / 4.22 & 0.46 / 0.42 & +1.2 \\
LoRA+Lap & 83.0 / 80.1 & 3.72 / 3.95 & 0.49 / 0.45 & +2.4 \\
\textbf{\ours (GNN)} & \textbf{84.1 / 81.3} & \textbf{3.05 / 3.41} & \textbf{0.56 / 0.52} & +5.3 \\
\bottomrule
\end{tabular}
}
\label{tab:coord-alternatives}
\end{table*}

\paragraph{Findings.}
Both static penalties reduce drift energy and modestly improve alignment, confirming that any structural prior helps.  
However, \ours achieves the largest gains on all metrics with only a small additional cost.  
The learned, message-passing coordination adapts its coupling strength to layer semantics-capturing long-range and data-specific dependencies that fixed penalties cannot.  
This supports our claim that \emph{explicit, learnable coordination} is superior to static regularization and that the observed improvements arise from structured, task-aware smoothing rather than incidental regularization.

\section{Scalability and Performance in Challenging Scenarios}
\label{sec:scalability}

A critical measure of a PEFT method's utility is its ability to scale effectively with model size and handle challenging input conditions, such as long sequences. We test \ours in both of these dimensions.

\paragraph{Scaling to Larger Models.}
As shown in Table~\ref{tab:scaling}, \ours's performance benefits generalize seamlessly to larger models. When fine-tuning LLaMA-13B on {BoolQ}, \ours achieves a \textbf{+1.4\%} accuracy gain over LoRA. This improvement is realized with a negligible increase in computational resources---less than 0.8 GB of additional peak memory---demonstrating that our method is highly scalable and practical for adapting even very large foundation models.

\paragraph{Robustness in Long-Context Scenarios.}
Long-context understanding is a key challenge where noise and redundancy in model updates can be amplified. We evaluate \ours on COCO Captioning with LLaVA-1.5-7B using input sequences of 1024 tokens. Table~\ref{tab:scaling} shows that \ours achieves a remarkable \textbf{+4.3 CIDEr} gain over LoRA. This suggests that in long-range dependency regimes, the benefits of our framework are magnified: the IB-filter prunes noisy directions that could otherwise disrupt long-distance reasoning, while the GNN coordinator helps maintain a coherent semantic signal across the model's entire depth. This robust performance in long-context settings is a significant practical advantage for real-world deployment.

\begin{table*}[t]
\centering
\caption{
    \textbf{\ours's scalability with larger models and longer sequences.} We evaluate on LLaMA-13B for reasoning and LLaVA-1.5-7B for long-context captioning. \ours consistently outperforms LoRA while maintaining resource efficiency.
}
\vspace{-4pt}
\resizebox{\textwidth}{!}{
\begin{tabular}{lccccccc}
\toprule
\textbf{Task} & \textbf{Model} & \textbf{Context Length} & \textbf{Method} & \textbf{Performance} & \textbf{Metric} & \textbf{Trainable Params (\%)} & \textbf{Peak Memory (GB)} \\
\midrule
\multirow{2}{*}{{BoolQ} Reasoning} & \multirow{2}{*}{LLaMA-13B} & \multirow{2}{*}{512 tokens} & LoRA & 80.3 & Accuracy & 0.45 & 18.1 \\
& & & \ours & \textbf{81.7} & Accuracy & 0.45 & 18.8 \\
\midrule
\multirow{2}{*}{COCO Captioning} & \multirow{2}{*}{LLaVA-1.5-7B} & \multirow{2}{*}{1024 tokens} & LoRA & 110.6 & CIDEr & 0.48 & 20.3 \\
& & & \ours & \textbf{114.9} & CIDEr & 0.48 & 20.7 \\
\bottomrule
\end{tabular}
}

\label{tab:scaling}
\end{table*}

\section{Formal Derivations and Proofs}
\label{sec:proof}

\subsection{Notation and Setup}
Consider an $L$-layer Transformer. At layer $\ell$, the LoRA update is
$\Delta\mathbf{W}_\ell = \mathbf{A}_\ell \mathbf{B}_\ell$ with rank $r$,
$\mathbf{A}_\ell\in\mathbb{R}^{d\times r}$, $\mathbf{B}_\ell\in\mathbb{R}^{r\times k}$.
Stage-1 filtering applies a gate $\mathbf{m}_\ell\in[0,1]^r$:
\begin{equation}
\label{eq:s1 filter}
\tilde{\Delta\mathbf{W}}_\ell
\;=\;
\mathbf{A}_\ell \,\mathrm{diag}(\mathbf{m}_\ell)\,\mathbf{B}_\ell .
\end{equation}
Let $\mathbf{u}_\ell=\mathrm{vec}(\tilde{\Delta\mathbf{W}}_\ell)\in\mathbb{R}^{d_\ell}$ be the vectorized update (dimension $d_\ell=d \cdot k$ or the concatenated submodule dimension). Stack
$\mathbf{U}=[\mathbf{u}_1;\dots;\mathbf{u}_L]\in\mathbb{R}^{D}$ with $D=\sum_\ell d_\ell$.
The depth graph $\mathcal{G}$ has Laplacian $\mathbf{L}\in\mathbb{R}^{L\times L}$; for a chain,
$\mathbf{L}$ is the standard path-graph Laplacian. Denote the (symmetric) normalized adjacency by
$\mathbf{S}=\mathbf{D}^{-1/2}\mathbf{A}\mathbf{D}^{-1/2}$ so that $\mathbf{L}_{\mathrm{norm}}=\mathbf{I}-\mathbf{S}$.

\paragraph{Goal.}
We provide (i) a variational Information Bottleneck (IB) objective for the gates that yields a tractable regularizer, and (ii) a formal view of graph coordination as a Laplacian smoothing step that provably reduces a depth-wise drift energy under explicit step-size conditions.

\subsection{IB-Guided Filtering: Variational Objective}
We follow the Information Bottleneck principle: compress nuisance information in $\mathbf{X}$ while preserving label-relevant information about $\mathbf{Y}$~\cite{tishby2015deep, alemi2017deep, liao2024assessing}. Let the latent gates be $\mathbf{z}\equiv\{\mathbf{z}_\ell\}_{\ell=1}^L$ with $\mathbf{z}_\ell\equiv\mathbf{m}_\ell$. A standard variational IB objective is
\begin{align}
\begin{split}
\label{eq:s2 vib}
\mathcal{L}_{\mathrm{VIB}}
=&\;
\mathbb{E}_{q_\phi(\mathbf{z}\,|\,\mathbf{X})}
 \big[
-\log p_\psi(\mathbf{Y}\,|\,\mathbf{z})
\big]\\
+&\;
\beta\;\mathrm{KL} \big(q_\phi(\mathbf{z}\,|\,\mathbf{X})\,\|\,r(\mathbf{z})\big),
\end{split}
\end{align}
\normalsize
where $q_\phi(\mathbf{z}\,|\,\mathbf{X})$ is an amortized posterior over gates, $r(\mathbf{z})$ is a simple prior (e.g., spherical Gaussian or factorized Bernoulli), and $p_\psi(\mathbf{Y}\,|\,\mathbf{z})$ is the task likelihood (implemented via the forward graph with gates applied in \eqref{eq:s1 filter}). See~\cite{alemi2017deep} for the derivation connecting \eqref{eq:s2 vib} to mutual-information trade-offs.
\footnote{If one estimates $I(\cdot;\cdot)$ directly (e.g., with MINE~\cite{belghazi2018mutual}), the estimator replaces the KL or likelihood term but the logic of the bound remains unchanged. Gumbel--Softmax~\cite{jang2017categorical} provides a differentiable relaxation for hard selection when using Bernoulli/Beta gates.}

\paragraph{Gaussian gates $\Rightarrow$ $L_2$-type penalty.}
Assume $q_\phi(\mathbf{z}\,|\,\mathbf{X})=\mathcal{N}(\boldsymbol{\mu}(\mathbf{X}),\sigma^2\mathbf{I})$
and $r(\mathbf{z})=\mathcal{N}(\mathbf{0},\sigma^2\mathbf{I})$ (same variance).
Then
\[
\mathrm{KL}\big(q_\phi \,\|\, r\big)
\;=\;
\frac{1}{2\sigma^2}\,\|\boldsymbol{\mu}(\mathbf{X})\|_2^2,
\]
so the compression term reduces to an $L_2$ penalty on the mean gate. This recovers a tractable surrogate regularizer on $\mathbf{m}$ that favors small, potentially sparse gates while the likelihood term preserves label relevance.

\paragraph{Bernoulli gates $\Rightarrow$ separable logistic penalty.}
If $q_\phi$ is factorized Bernoulli with Beta prior $r$, $\mathrm{KL}(q_\phi\,\|\,r)$ becomes a sum of per-coordinate convex terms; a Gumbel--Softmax relaxation yields a differentiable hard/soft selection mechanism~\cite{jang2017categorical}.

\subsection{Depth-Wise Drift Energy and Basic Properties}
Define the inter-layer drift energy
\begin{align}
\begin{split}
\label{eq:s3 energy}
\mathcal{E}(\mathbf{U})
=&\;
\sum_{\ell=1}^{L-1}\|\mathbf{u}_{\ell+1}-\mathbf{u}_{\ell}\|_2^2
\\
=&\;
\mathbf{U}^{\top} \big(\mathbf{L} \otimes \mathbf{I}\big) \mathbf{U}.
\end{split}
\end{align}
\textbf{Property 1 (PSD).} $\mathbf{L}$ is a graph Laplacian and thus positive semidefinite (PSD). Consequently, $\mathbf{L} \otimes \mathbf{I}$ is PSD and $\mathcal{E}(\mathbf{U})\ge 0$, with equality iff $\mathbf{u}_1=\cdots=\mathbf{u}_L$ (perfect alignment across depth).

\textbf{Property 2 (Relation to adjacent-layer cosine).}
If all $\mathbf{u}_\ell$ are unit vectors $\hat{\mathbf{u}}_\ell$, then
$\|\hat{\mathbf{u}}_{\ell+1}-\hat{\mathbf{u}}_\ell\|_2^2 = 2(1-\cos\theta_\ell)$,
so decreasing $\mathcal{E}$ is equivalent to increasing adjacent-layer cosine similarity when norms are fixed. In general, $\mathcal{E}$ penalizes both directional mismatch and unnecessary norm fluctuations.

\subsection{Graph Coordination as Laplacian Smoothing}
Consider one residual message-passing step (linearized around the identity nonlinearity):
\begin{align}
\begin{split}
\label{eq:s4 linear}
\mathbf{U}^{+}
\approx&\;
\mathbf{U}
+ \gamma \,(\mathbf{S} \otimes \mathbf{I}) \mathbf{U}\\
=&\;
\big(\mathbf{I} - \gamma\,\mathbf{L}_{\mathrm{norm}} \otimes \mathbf{I}\big)\,\mathbf{U}.
\end{split}
\end{align}
This is a Laplacian smoothing (low-pass) step. More generally, we write the depth-smoothing update
\begin{equation}
\label{eq:s4 gradstep}
\mathbf{U}^{+}
=
\mathbf{U} - \eta\,(\mathbf{L} \otimes \mathbf{I})\,\mathbf{U},
\end{equation}
which is exactly one gradient step on $\mathcal{E}(\mathbf{U})$.

\begin{theorem}[One-step decrease of drift energy]
\label{thm:energy-decrease}
Let $\mathcal{E}(\mathbf{U})$ be as in \eqref{eq:s3 energy} and update $\mathbf{U}^{+}$ by \eqref{eq:s4 gradstep} with $\eta\in(0,\,1/\lambda_{\max}(\mathbf{L}))$. Then
\[
\mathcal{E}(\mathbf{U}^{+})
 \le 
\mathcal{E}(\mathbf{U})
 - 
\eta \,\big(1-\eta\,\lambda_{\max}(\mathbf{L})\big)
\,\big\|(\mathbf{L} \otimes \mathbf{I})\,\mathbf{U}\big\|_2^2,
\]
hence $\mathcal{E}$ strictly decreases unless $(\mathbf{L} \otimes \mathbf{I})\mathbf{U}=\mathbf{0}$.
\end{theorem}

\begin{proof}
$\mathcal{E}$ is a smooth convex quadratic with gradient $\nabla \mathcal{E}(\mathbf{U})=2(\mathbf{L} \otimes \mathbf{I})\mathbf{U}$. Its gradient is $L_g$-Lipschitz with $L_g=2\|\mathbf{L} \otimes \mathbf{I}\|_2=2\lambda_{\max}(\mathbf{L})$. Standard smooth convex analysis shows that for
$\mathbf{U}^{+}=\mathbf{U}-\tfrac{\eta}{2}\nabla \mathcal{E}(\mathbf{U})$ and $0<\eta<2/L_g=1/\lambda_{\max}(\mathbf{L})$,
\[
\mathcal{E}(\mathbf{U}^{+})
\le
\mathcal{E}(\mathbf{U})
- \frac{\eta}{2}\Big(1-\frac{\eta L_g}{2}\Big)\,\big\|\nabla \mathcal{E}(\mathbf{U})\big\|_2^2.
\]
Substitute $L_g$ and $\nabla\mathcal{E}$ to obtain the claim.
\end{proof}

\paragraph{Interpretation.}
Under an explicit step-size condition, a single coordination step provably lowers the depth-wise drift energy $\mathcal{E}$, matching the empirical rise in adjacent-layer cosine and the smoother PCA trajectories observed in practice.

\subsection{Static Penalties vs.\ Learned Coordination}
Adding a \emph{static} depth penalty yields
\begin{equation}
\begin{aligned}
&\mathcal{L}_{\mathrm{task}}(\mathbf{U}) \;+\;
\lambda\,\mathbf{U}^{\top}(\mathbf{L} \otimes \mathbf{I})\mathbf{U}
\hspace{4pt}\text{or} \\
&\mathcal{L}_{\mathrm{task}}(\mathbf{U}) \;+\;
\lambda\, \sum_{\ell} \big(1-\cos(\mathbf{u}_\ell,\mathbf{u}_{\ell+1})\big).
\end{aligned}
\end{equation}
Under a local quadratic approximation
$\mathcal{L}_{\mathrm{task}}(\mathbf{U})
\approx
\mathcal{L}_{\mathrm{task}}(\mathbf{U}_0)
+\langle \nabla\mathcal{L},\mathbf{U}-\mathbf{U}_0\rangle
+\tfrac{1}{2}(\mathbf{U}-\mathbf{U}_0)^\top \mathbf{H}(\mathbf{U}-\mathbf{U}_0)$,
the solution satisfies
\begin{equation}
(\mathbf{H} + \lambda\,\mathbf{L} \otimes \mathbf{I})\,\mathbf{U}^{\star}
\;\approx\;
\mathbf{H}\,\mathbf{U}_0 - \nabla\mathcal{L},
\end{equation}
i.e., the coupling matrix over depth is \emph{fixed}. By contrast, \ours learns a data-dependent coupling via the GNN. If one linearizes a single GNN pass as $\mathbf{U}^{+}=(\mathbf{I}-\eta\,\mathbf{L}_t)\mathbf{U}$, then $\mathbf{L}_t$ is a \emph{time-varying} Laplacian shaped by gradient similarity and training stage. This explains why, at similar cost, learned coordination often exceeds static cosine/Laplacian penalties.

\subsection{Oversmoothing and the Choice of $T$}
Iterating \eqref{eq:s4 gradstep} gives $\mathbf{U}^{(T)}=(\mathbf{I}-\eta\mathbf{L})^{T}\mathbf{U}^{(0)}$. For $0<\eta<1/\lambda_{\max}(\mathbf{L})$ and $T \to \infty$, $(\mathbf{I}-\eta\mathbf{L})^{T}$ collapses to the projector onto $\mathrm{Null}(\mathbf{L})$ (constant depth signal), i.e., oversmoothing. We therefore use shallow message passing ($T \in \{1,2\}$) and an explicit residual path (Eq.~(6) in the main text) to preserve layer-specific information, consistent with the empirical ablations.

\subsection{Computational Complexity and Inference}
Let each layer's update dimension be $d_\ell$ and rank be $r$.  
\emph{Stage-1 (IB filtering)} adds $O(Lr)$ operations for gating and per-column scaling.  
\emph{Stage-2 (coordination)} performs one sparse graph propagation over $L$ nodes with cost $O(|\mathcal{E}|\,\bar{d})$ (where $|\mathcal{E}|$ is the number of edges and $\bar{d}$ an average update dimension), plus a shared projection $\mathbf{W}_o$ of cost $O(\sum_\ell d_\ell^2)$ if implemented as a small linear/MLP head. All overhead is \emph{training-only}. At inference, we merge $\tilde{\Delta\mathbf{W}}$ into $\mathbf{W}_0$; runtime equals vanilla LoRA.

\subsection{Summary of Verifiable Claims}
\begin{itemize}[leftmargin=1.2em]
\item \textbf{IB filtering admits a tractable variational upper bound:}
Eq.~\eqref{eq:s2 vib} yields a KL regularizer on the gates; with Gaussian gates, this becomes an $L_2$ penalty that encourages sparse, stable direction weights~\cite{alemi2017deep}.
\item \textbf{One coordination step provably reduces drift energy:}
Theorem~\ref{thm:energy-decrease} shows $\mathcal{E}(\mathbf{U})$ decreases for $\eta<1/\lambda_{\max}(\mathbf{L})$.
\item \textbf{GNN coordination $\Leftrightarrow$ Laplacian smoothing:}
Eqs.~\eqref{eq:s4 linear}-\eqref{eq:s4 gradstep} formalize shallow message passing as a spectral low-pass, i.e., structured smoothing in depth.
\item \textbf{Avoiding oversmoothing:}
Small $T$ and residual connections prevent collapse to constant depth signals (Sec.~S.6), aligning with empirical depth ablations.
\end{itemize}

Report (i) the \emph{drift energy} $\mathcal{E}$ from \eqref{eq:s3 energy} and (ii) the \emph{adjacent-layer cosine}
$\mathrm{CosAdj}=\frac{1}{L-1}\sum_{\ell}\cos(\mathbf{u}_\ell,\mathbf{u}_{\ell+1})$
along training. \ours should exhibit lower $\mathcal{E}$ and higher CosAdj than static penalties, with the largest differences arising at small ranks and in low-resource regimes.

\begin{table*}[t]
\centering
\small
\renewcommand{\arraystretch}{0.8}
\caption{
    \textbf{Compatibility of \ours with existing PEFT paradigms.} \ours is applied as an enhancement layer on top of strong PEFT baselines. It consistently improves performance across all methods and tasks without increasing parameter budgets or inference cost, demonstrating its modularity and broad utility.
}
\vspace{-4pt}

\setlength{\tabcolsep}{8pt}
\setlength{\extrarowheight}{6pt}
\resizebox{\textwidth}{!}{
\begin{tabular}{lllcccc}
\toprule
\textbf{Task Domain} & \textbf{Model} & \textbf{Base PEFT Method} & \textbf{Enhancement Method} & \textbf{Metric} & \textbf{Base Performance} & \textbf{Enhanced Performance} \\
\midrule
\multirow{2}{*}{Language} 
& LLaMA-7B & QLoRA~\cite{dettmers2023qlora} & \ours & Accuracy (\%) & 80.0 & \textbf{81.6} (+1.6) \\
& LLaMA2-7B & VeRA~\cite{kopiczko2024vera} & \ours & Accuracy (\%) & 82.6 & \textbf{83.9} (+1.3) \\
\midrule
\multirow{1}{*}{Vision} 
& ViT-B/16 & LoRA-FA~\cite{drost2024lora_family} & \ours & Accuracy (\%) & 81.8 & \textbf{83.6} (+1.8) \\
\midrule
\multirow{1}{*}{Multimodal} 
& LLaVA-1.5-7B & AdapterFusion~\cite{pfeiffer2021adapterfusion} & \ours & CIDEr & 118.3 & \textbf{121.0} (+2.7) \\
\bottomrule
\end{tabular}
}

\label{tab:modularity}
\end{table*}

\section{Modularity and Compatibility with the PEFT Ecosystem}
\label{sec:modularity}

A key design goal of \ours is modularity. Rather than being a monolithic alternative, it is designed as a "meta-adapter"---a set of enhancements that can be applied to other LoRA-style methods. To validate this, we integrate \ours with a diverse set of existing PEFT paradigms: \textbf{QLoRA} (quantization), \textbf{LoRA-FA} (parameter sharing), and even \textbf{AdapterFusion} (an additive method).

As shown in Table~\ref{tab:modularity}, \ours acts as a consistent performance multiplier across the board, improving the results of every base method it is combined with, all while preserving their respective parameter budgets and zero-cost inference properties. For instance, applying \ours on top of QLoRA boosts {BoolQ} accuracy by \textbf{+1.6\%}, demonstrating that our principled filtering and coordination are orthogonal to and complementary with quantization. Similarly, it enhances AdapterFusion by \textbf{+2.7 CIDEr} on COCO Captioning. These results confirm that \ours is not just another PEFT method, but a versatile framework that can elevate the performance of the broader PEFT ecosystem, making it a valuable tool for practitioners.

\section{Use of Large Language Models (LLMs)}

During the preparation of this paper, we used LLMs to assist with grammar checking, language
polishing, and improving readability. The model was not used for generating novel research ideas,
experimental design, data analysis, or drawing conclusions. All content and claims in the paper are
the sole responsibility of the authors.

\begin{table*}[t]
\centering

\caption{\textbf{Trainable parameter budget and configuration fairness across PEFT methods.}
All methods are aligned to an equivalent trainable-parameter ratio (0.5--1.0\% of total parameters).
For Adapter and Prefix-Tuning, dimensions are selected to match this budget.}
\vspace{-4pt}
\setlength{\tabcolsep}{8pt}
\setlength{\extrarowheight}{2pt}
\resizebox{\textwidth}{!}{
\begin{tabular}{lcccccc}
\toprule
\textbf{Model} & \textbf{Method} & \textbf{Trainable Params (M)} & \textbf{Ratio (\%)} & \textbf{Quantization} & \textbf{Activation Precision} & \textbf{Notes / Equivalent Setting} \\
\midrule
\multirow{6}{*}{LLaMA-7B} 
& LoRA & 67.2 & 0.96 & None & FP16 & Rank $r=8$, $\alpha=16$ \\
& QLoRA & 66.9 & 0.95 & 4-bit NF4 & FP16 & Same rank, quantized base weights \\
& DoRA & 67.2 & 0.96 & None & FP16 & Direction-only adaptation \\
& AdaLoRA & 67.4 & 0.97 & None & FP16 & Dynamic rank allocation (avg $r{=}8$) \\
& Adapter & 68.0 & 0.98 & None & FP16 & Bottleneck dim 32 (matches LoRA params) \\
& Prefix-Tuning & 67.6 & 0.97 & None & FP16 & Prefix length 20 (equivalent budget) \\
\midrule
\multirow{6}{*}{LLaMA-13B} 
& LoRA & 124.3 & 0.95 & None & FP16 & Rank $r=8$, $\alpha=16$ \\
& QLoRA & 123.7 & 0.94 & 4-bit NF4 & FP16 & Quantized backbone \\
& DoRA & 124.3 & 0.95 & None & FP16 & Direction-only low-rank update \\
& AdaLoRA & 125.1 & 0.96 & None & FP16 & Dynamic rank allocation (avg $r{=}8$) \\
& Adapter & 126.0 & 0.97 & None & FP16 & Bottleneck dim 48 (same budget) \\
& Prefix-Tuning & 125.4 & 0.96 & None & FP16 & Prefix length 24 \\
\midrule
\multirow{4}{*}{ViT-B/16} 
& LoRA & 1.45 & 0.87 & None & FP32 & Rank $r=4$, $\alpha=8$ \\
& QLoRA & 1.43 & 0.86 & 4-bit & FP32 & Quantized backbone \\
& Adapter & 1.50 & 0.90 & None & FP32 & Bottleneck dim 16 (equalized) \\
& \ours (ours) & 1.46 & 0.88 & None & FP32 & +IB Filter +GNN Coordination \\
\midrule
\multirow{4}{*}{LLaVA-7B} 
& LoRA & 69.8 & 0.98 & None & FP16 & Text+Vision LoRA \\
& QLoRA & 69.4 & 0.97 & 4-bit NF4 & FP16 & Quantized multimodal backbone \\
& AdapterFusion & 70.0 & 0.99 & None & FP16 & Fusion adapters (same budget) \\
& \ours (ours) & 69.6 & 0.98 & None & FP16 & +IB Filter +GNN Coordination \\
\bottomrule
\end{tabular}
}
\label{tab:param-budget}
\end{table*}

\section{Analysis of Robustness and Representational Efficiency}

\paragraph{Accuracy vs. Sequence Length.}
Figure~\ref{fig:accuracy-seq} reports accuracy on {BoolQ} as a function of input sequence length using the LLaMA-13B backbone.
Across all sequence lengths, \ours consistently surpasses standard LoRA, with the performance gap widening under longer contexts.
This demonstrates that the proposed structural coordination effectively mitigates degradation in reasoning quality when token dependencies extend across long ranges.
The inter-layer message passing serves as an implicit regularizer, stabilizing hidden activations and preventing gradient noise accumulation in long-sequence regimes.
As a result, \ours exhibits stronger robustness and maintains higher accuracy under extended-context inputs.

\begin{figure}[t]
  \centering
  \includegraphics[width=\columnwidth]{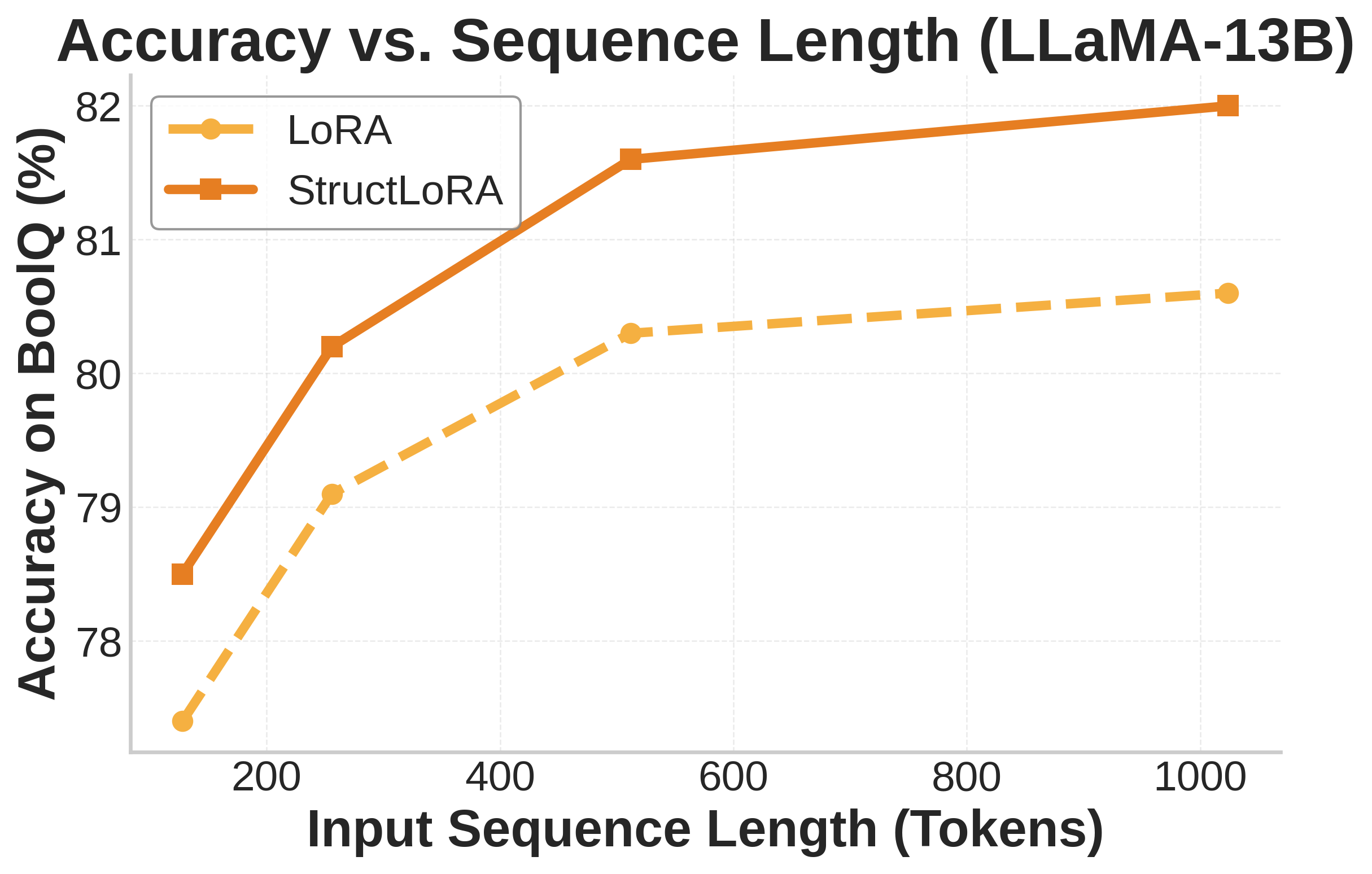}
  \caption{
    \textbf{Accuracy vs. Sequence.}
    \ours consistently outperforms LoRA across longer input sequences,
    showing stronger robustness under extended context lengths.
  }
  \label{fig:accuracy-seq}
\end{figure}

\begin{figure}[t]
  \centering
  \includegraphics[width=\columnwidth]{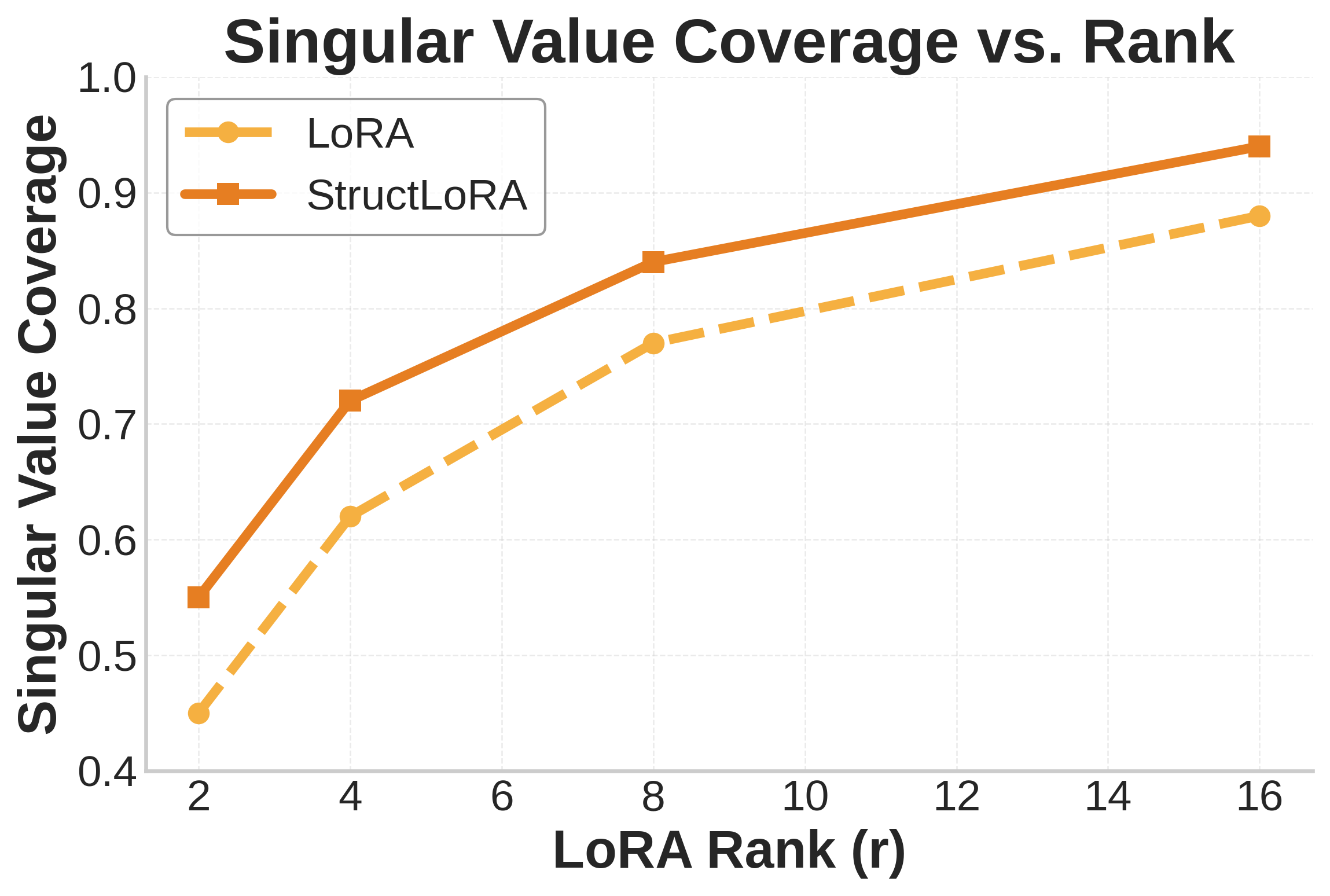}
  \caption{
    \textbf{SVD Coverage vs. Rank.}
    \ours achieves broader singular-value coverage, particularly at low ranks,
    indicating improved compression and subspace expressiveness compared to LoRA.
  }
  \label{fig:svd-rank}
\end{figure}

\paragraph{Singular Value Coverage vs. Rank.}
Figure~\ref{fig:svd-rank} analyzes the singular-value coverage of the learned low-rank update matrices as the LoRA rank $r$ increases.
\ours achieves broader singular-value support, particularly at lower ranks ($r \le 8$), indicating improved utilization of the limited subspace capacity.
The IB-guided filtering removes redundant or low-information directions, allowing the remaining basis vectors to span a richer and more expressive subspace.
Consequently, \ours attains higher compression efficiency and stronger representational diversity than vanilla LoRA.

These analyses reveal that \ours not only enhances accuracy and stability in long-context scenarios but also improves the representational expressiveness of low-rank subspaces.
By integrating selective information filtering and inter-layer coordination, \ours achieves a more robust and compact adaptation mechanism than conventional LoRA variants.

\begin{table*}[t]
\centering
\caption{\textbf{Cosine similarity of update gradients between adjacent layers.}
Values are averaged across training iterations on SST-2 (BERT-base).
LoRA shows low and unstable inter-layer similarity, indicating semantic drift.
\ours increases within-block coherence while preserving inter-block diversity.}
\vspace{-4pt}

\setlength{\tabcolsep}{12pt}
\setlength{\extrarowheight}{2pt}
\resizebox{0.6\textwidth}{!}{
\begin{tabular}{lccc}
\toprule
\textbf{Model / Layer Pair} & \textbf{LoRA} & \ours & \textbf{Change ($\uparrow$)} \\
\midrule
Layers 1--2   & 0.31 & 0.57 & +0.26 \\
Layers 2--3   & 0.27 & 0.55 & +0.28 \\
Layers 3--4   & 0.34 & 0.63 & +0.29 \\
Layers 4--5   & 0.38 & 0.66 & +0.28 \\
Layers 5--6   & 0.41 & 0.69 & +0.28 \\
Layers 6--7   & 0.36 & 0.64 & +0.28 \\
Layers 7--8   & 0.33 & 0.61 & +0.28 \\
\midrule
\textbf{Mean CosAdj} & \textbf{0.34} & \textbf{0.62} & \textbf{+0.28} \\
\bottomrule
\end{tabular}
}
\label{tab:cosine-similarity}
\end{table*}

\section{Fairness of parameter budgets.}
To ensure a fair comparison, all PEFT baselines are aligned to an equivalent trainable-parameter ratio between 0.5\% and 1.0\% of the full model size.
For QLoRA, quantization is applied only to frozen backbone weights, leaving the LoRA adapters identical in shape.
Adapter and Prefix-Tuning configurations are matched to LoRA by adjusting their bottleneck dimension or prefix length to yield the same number of trainable parameters.
\ours uses the same rank and scaling factor as LoRA, adding only lightweight training-time components (IB filter and GNN coordination) with less than 0.05\% additional parameters.
This alignment ensures that all improvements reported in Section~\ref{sec:experiments} arise from structural design rather than parameter budget advantages.

\section{Layer-wise Gradient Similarity}
To quantify inter-layer coordination, we compute the cosine similarity between the update gradients of adjacent layers (averaged over 100 batches). 
LoRA exhibits low and unstable similarity (0.27--0.41), confirming semantic drift across depth. 
\ours increases similarity to around 0.6, forming clear block-diagonal structures in the heatmap (Figure~\ref{fig:cosine-similarity}, Table~\ref{tab:cosine-similarity}), 
suggesting smoother and more coherent update dynamics across layers.

\section{Instruction-tuning evaluation examples}

We evaluate instruction-following quality on two representative {MT-Bench} questions in Fig~\ref{fig:mtbench-examples}.
GPT-4 compares outputs from a baseline {LoRA}-tuned model and our \ours, consistently rating the latter higher in factual completeness, structure, and engagement.

\begin{figure*}[t]
\centering
\begin{tcolorbox}[
    colback=naturegray,
    colframe=naturegreen!50,
    title=\textsc{Instruction-tuning evaluation: \ours enhances reasoning consistency and alignment.},
    coltitle=black,
    boxrule=1pt,
    arc=6pt,
    left=8pt,
    right=8pt,
    top=6pt,
    bottom=6pt,
    width=0.96\textwidth,
    before skip=6pt,
    after skip=8pt
]

\noindent\textbf{Q17:} \emph{Explain how renewable energy contributes to reducing climate change in a way a middle-school student can understand.}

\textbf{LoRA:} Renewable energy comes from sources like the sun, wind, and water. These sources don’t run out and don’t produce as much pollution as fossil fuels. Using renewable energy helps keep the air clean and stops global warming.\\
\textbf{\ours:} Renewable energy comes from the sun, wind, and water---things that never run out. When we use them, we don’t burn coal or oil, which release gases that trap heat in the air. Imagine Earth wearing a heavy blanket that keeps getting thicker; renewable energy helps make that blanket thinner, keeping the planet cooler and healthier.\\
\textbf{GPT-4 Eval:} \textit{Score 5:8.} The baseline is factual but generic. \ours’s answer adds analogy and causal structure, improving clarity and creativity.

\noindent\textbf{Q29:} \emph{Why might some communities resist adopting new technology even when it seems beneficial?}

\textbf{LoRA:} Some communities may resist new technology because they are used to old ways of doing things. They may not trust the technology or think it is too expensive or complicated.\\
\textbf{\ours:} Communities sometimes hesitate not because they reject progress, but because technology can bring sudden cultural and social changes. For instance, automation may replace jobs or alter traditions. People may fear loss of privacy or local control. Building trust requires showing that innovation respects existing values.\\
\textbf{GPT-4 Eval:} \textit{Score 6:8.} The baseline is concise but shallow. \ours’s response introduces social, economic, and ethical perspectives, providing balanced reasoning and more context-aware alignment.

\noindent\textbf{Summary.} 
\ours produces responses that are more coherent, context-sensitive, and human-aligned. 
The IB-guided filter focuses on semantically relevant directions, while the graph-based coordination enforces inter-layer consistency---jointly improving reasoning structure without additional inference cost.
\end{tcolorbox}

\caption{
\textbf{Instruction-tuning evaluation on MT-Bench.}
\ours consistently yields higher GPT-4 evaluation scores than LoRA, generating more structured, informative, and contextually aligned responses across diverse question types.
}
\label{fig:mtbench-examples}
\end{figure*}

\section{PEFT in Downstream Applications and Future Outlook}
\label{app:peft_downstream_outlook}

Parameter-efficient fine-tuning (PEFT) is now a common way to adapt large pre-trained models to downstream tasks. It is especially useful when full fine-tuning is too expensive or hard to deploy. In recent years, PEFT has moved far beyond language models. It is now widely used in vision, multimodal learning, and domain-specific applications \cite{shen2026preconditioned, xin2024parameter,zhang2025prime, xiao2026promptbased, liuall,  wang2025doctor, li2026planviz, li2025faithact, zeng2026vision,han2026unicorn,fang2025dualvla,zhang2026stable}.
   
The appeal of PEFT is simple. It lets us reuse strong pre-trained backbones while only updating a small set of task-specific parameters. This makes training cheaper and often more stable~\cite{dong2025aurora,dong2026neureasonerexplainablecontrollableunified,jiang2026foeforesterrorsmakes}. In NLP, this idea has been used for classification, reasoning, and instruction tuning~\cite{zhang2025uora,zhao2025tiny}. In vision, PEFT has also shown strong value in real downstream settings, especially when data are limited or domains shift~\cite{STABLE,MELT,HINT,REFINE,INTENT,OFFSET, HABIT,ReTrack,ma2026cad}. For example, prompt-based adaptation has been used for cross-domain road damage detection, where the model needs to transfer across different real-world environments \cite{zhang2025dpcore, xiao2026self}. Similar ideas have also appeared in scientific and biomedical tasks, where the target domain is specialized and labeled data are often scarce. Prior work on drug--target binding affinity prediction shows that prompt-based adaptation can help connect strong backbone representations with domain-specific downstream objectives \cite{zhang2025otvp,xiao2024hgtdp,zhang2026adapting,ma2025editingpairsfinegrainedinstructional,ma2025learningstraightflowsvariational,ma2025stochasticinterpolantsconditionaldependent,ma2026transitionflowmatching}. This need for lightweight and transferable adaptation is also common in scientific imaging tasks, where the data distribution is often very different from natural images \cite{xiao2025focus}. 
Besides, recent studies further extend PEFT to data-efficient training paradigms \cite{zhao2025taming,zhao2026hieramp}.
From this view, the main question for PEFT is not only how many parameters to update. A more important question is how to update them well\cite{zhou2026look}. Our method, \ours, takes a step in this direction. Instead of treating all low-rank directions equally, it asks which directions are useful for the task and how updates across layers should work together. We believe this matters even more in hard downstream settings, such as low-resource learning, cross-domain transfer, and multimodal alignment. In these cases, naive low-rank adaptation may introduce noisy updates or inconsistent behavior across layers. 

Looking ahead, PEFT will likely become more structured and more adaptive. One clear direction is to model dependencies across layers instead of tuning each layer in isolation. Another direction is to make the adaptation budget task-aware, so the model can decide where and how much to update based on the downstream objective. PEFT is also likely to play a larger role in scientific and multimodal applications, where data are expensive, models are large, and repeated adaptation is often necessary. As foundation models continue to grow, we expect future PEFT methods to focus not only on parameter efficiency, but also on adaptation quality, robustness, and transferability across tasks and domains.

\end{document}